\documentclass{article}
\PassOptionsToPackage{authoryear,square}{natbib}
\usepackage[main,preprint]{neurips_2026}
\usepackage[utf8]{inputenc}
\usepackage[T1]{fontenc}
\usepackage{hyperref}
\hypersetup{hidelinks,
  pdftitle={Centralized Critics for N-Agent Ad Hoc Teamwork},
  pdfauthor={Andrew McAlinden, Caroline Wang, Steven Shi, Arrasy Rahman, Elad Liebman, Peter Stone}}
\usepackage{url}
\usepackage{booktabs}
\usepackage{amsfonts,amssymb,amsmath,amsthm,mathtools,mathrsfs}
\usepackage{graphicx}
\usepackage{subcaption}
\usepackage{microtype}
\usepackage{xcolor}

\theoremstyle{plain}
\newtheorem{theorem}{Theorem}[section]

\newtheorem{corollary}[theorem]{Corollary}

\theoremstyle{definition}
\newtheorem{definition}[theorem]{Definition}

\newtheorem{assumption}[theorem]{Assumption}
\theoremstyle{remark}

\newcommand{\caroline}[1]{}
\title{Centralized Critics for $N$-Agent Ad Hoc Teamwork}
\author{Andrew McAlinden \And Caroline Wang \And Steven Shi \AND
Arrasy Rahman \And Elad Liebman \And Peter Stone}

\begin{document}
\maketitle
\begin{abstract}

Centralized training with decentralized execution (CTDE) has become the dominant paradigm in multi-agent reinforcement learning (MARL), motivating extensive study of what privileged information should be made available to the critic during training. While prior work has examined the role of global state in fully cooperative settings, the design of centralized critics for $N$-agent ad hoc teamwork---in which dynamically varying set of learning agents must coordinate with an arbitrary set of uncontrolled teammates drawn from an unknown distribution---remains poorly understood. We study this question by comparing two natural classes of advantage estimators: \emph{controlled history} estimators, which condition only on the observation and action histories of the learner's own agents, and \emph{joint history} estimators, which additionally condition on the histories of the uncontrolled teammates. We prove that both estimators are unbiased and characterize the conditions under which the joint history estimate attains strictly lower variance than its controlled counterpart. Empirically, joint history critics match or outperform controlled history critics across a range of NAHT tasks, yielding improved sample efficiency and final returns. Together, these results provide both a theoretical justification and an empirical recommendation for incorporating uncontrolled agent histories into centralized critics for $N$-agent ad hoc teamwork.
\end{abstract}

\section{Introduction}
Cooperative multi-agent reinforcement learning (CMARL) is a paradigm for learning coordinated team policies on fully cooperative tasks, assuming full control over all agents \citep{gronauer2022multi}. In contrast, ad hoc teamwork (AHT) considers the challenge of training a single agent to collaborate with previously unknown teammates on fully cooperative tasks \citep{stone2010adhoc}.
The recently introduced N-agent ad hoc teamwork (NAHT) problem generalizes both paradigms by considering the scenario where a set of \textit{controlled} agents must collaborate with dynamically varying numbers and types of \textit{uncontrolled} teammates without the ability to pre-coordinate. The primary challenges of the NAHT problem are (1) dealing with a changing number of (un)controlled agents, and (2) generalizing to unseen uncontrolled teammate behavior \citep{wang2024naht}. 

For the NAHT problem, we assume that controlled agents only have access to controlled agent information, as the number of uncontrolled agents changes per episode, and the controlled agents must generalize to differing types of uncontrolled agents without being given direct information about their partners at inference time. Previous successful algorithms for NAHT have chosen to also include this restriction for critics, with algorithms \citep{wang2025shapleymachinegametheoreticframework, wang2024naht, wang2025matnaht} only including controlled agent information in critics. The CTDE paradigm then poses the question of if uncontrolled agent information can be helpful in training while still restricting this information during inference.

Since the purpose of critics is to estimate the advantage function, it follows that better advantage estimation should lead to better sample efficiency when training policies. We study this question by comparing two natural classes of advantage estimators: \emph{controlled history} estimators, which condition only on the observation and action histories of the learner's own agents, and \emph{joint history} estimators, which additionally condition on the histories of the uncontrolled teammates. We prove in Section~\ref{sec:theory} that both estimators are unbiased and characterize the conditions under which the joint history estimate attains strictly lower variance than its controlled counterpart. Empirically (Section~\ref{sec:experiments}), joint history critics match or outperform controlled history critics across a range of NAHT tasks, yielding improved sample efficiency and final returns. Together, these results provide both a theoretical justification and an empirical recommendation for incorporating uncontrolled agent histories into centralized critics for $N$-agent ad hoc teamwork.


\section{Related Work}
\label{sec:related-work}

Previous algorithms for NAHT rely only on controlled agent information in critics. There is no existing study on including privileged information in centralized critics for NAHT, although this kind of analysis exists for MARL. The most relevant previous work focuses on the variance of the Q function and policy gradients for different types of MARL critics.

\paragraph{Previous Algorithms for NAHT.}
POAM \citep{wang2024naht} was the first proposed algorithm for NAHT. It uses decentralized, RNN-based critics that attempt to model the value function using only information from an individual controlled agent. Similarly, the Shapley Machine algorithm \citep{wang2025shapleymachinegametheoreticframework} uses the POAM critic, but models a different value function per-agent. The MAT-NAHT algorithm \citep{wang2025matnaht} is the first to use a single attention-based centralized critic with information from all controlled agents. The same paper introduced Cent Obs POAM, which uses a single centralized RNN-based critic with information from all controlled agents. Since our work is focused on centralized critics, we use MAT-NAHT and Cent Obs Poam for experiments, and abbreviate both methods to ``MAT'' and ``POAM.''

\paragraph{Centralized Critics in MARL.}
A separate line of work has analyzed the role of privileged information in centralized critics for cooperative MARL. \citet{lyu2023centralized} provide an extensive theoretical and empirical study of history-based, state-based, and history-state-based critics for Dec-POMDPs. Among their results, they show that centralized history-based critics yield the same expected policy gradient as decentralized history-based critics but with strictly higher variance, that state-based critics may introduce unbounded bias when used to train history-based policies, and that history-state critics are unbiased but again carry higher variance than purely history-based critics. While their analysis covers what kind of \emph{global} information should be exposed to the critic in MARL, the NAHT setting introduces a different design choice: whether the critic should condition on the trajectories of agents that are \emph{not under the learner's control}. Our work extends the bias-variance style of analysis pioneered by \citet{lyu2023centralized} to this setting, deriving an unbiasedness result and a variance comparison between controlled-history and joint-history advantage estimators in the NAHT problem. Since we use PPO for optimization instead of the vanilla policy gradient, we focus on analyzing the bias and variance of advantage function.

\section{Background}

The N-agent ad hoc teamwork (NAHT) problem is formalized in the framework of Decentralized Partially Observable Markov Decision Processes (Dec-POMDPs) \citep{bernstein2002complexity}. A Dec-POMDP is described by the tuple $\langle M, S, \{O^i\}_{i=1}^M, A, T, r, \gamma \rangle$. Here, $M$ denotes the number of agents, $S$ the global state space, $O^i$ the local observation space for agent $i$, and $A = \prod_{i=1}^{M}A^i$ the joint action space. Let the notation $\Delta(\bullet)$ denote a probability distribution over the space $\bullet$. The transition function $T: S \times A \rightarrow \Delta(S)$ governs the dynamics, and the reward function $r: S \times A \rightarrow \mathbb{R}$ provides a shared reward signal.  In this work, we examine a control scheme where each agent selects actions according to the stochastic policy $\pi^i: H^i \times A^1 \times \cdots A^{i-1} \rightarrow \Delta(A^i)$, where $H^i$ is the fixed length history of observations and actions available to agent $i$. The decision-making problem occurs under the horizon $T$.

The NAHT problem considers a controlled set of agents $C$, and an uncontrolled set $U$. The learning algorithm generates policies for agents in $C$, while uncontrolled agents in $U$ follow fixed, potentially unknown policies. Each episode begins with a team sampling process $X(U, C)$, consisting of two steps: first, the number of controlled agents $N < M$ is sampled, and second, a set of $N$ agents is drawn from the controlled set $C$, while $M-N$ agents are drawn from the uncontrolled set $U$. The NAHT learning problem considers a parameterized set of controlled agents, $C(\theta)$, and aims to maximize the expected return under varying teammates and team compositions: $\max_{\theta} \mathbb{E}_{\{\pi^{i}\}_{i=1}^M \sim X(U, C(\theta))} \left[ \sum_{t=0}^{T} \gamma^t r_t \right].$

A \textbf{trajectory} $\tau = (s_0, a_0, r_0, s_1, a_1, r_1, \ldots, s_T)$ is sampled under the joint policy: at each timestep, the controlled agents select $a_t^c \sim \pi(\cdot \mid h_t^c)$ and the uncontrolled agents select $a_t^u \sim \pi^u(\cdot \mid h_t^u)$, with joint action $a_t = (a_t^c, a_t^u)$, rewards $r_t = R(s_t, a_t)$, and transitions $s_{t+1} \sim T(\cdot \mid s_t, a_t)$. We denote by $h_t^c$ and $h_t^u$ the observation--action histories of the controlled and uncontrolled agents, respectively, up to timestep $t$. Note that throughout all future sections, we abbreviate $a_t^c$ to $a_t$. The symbol $a_t$ denotes the controlled joint action, and $a_t^u$ denotes the uncontrolled joint action. Uncontrolled actions are sampled from $\pi^u$ and marginalized as part of the environment randomness.

We define action-value functions, value functions, and advantage functions at two levels of information access, differing in whether the centralized critic has access to the uncontrolled agents' histories.

The \textbf{joint history action-value function} (or Q-function) is a centralized action-value function conditioned on the full joint history of all agents. For a fixed policy $\pi$, it is defined as:
\begin{equation}
    Q^\pi(h_t^c, h_t^u, a_t) = \mathbb{E}_{\tau \sim \pi \mid h_t^c, h_t^u, a_t}\!\left[\sum_{k=0}^{T-t} \gamma^k\, r_{t+k}\right]
\end{equation}
where the expectation is with respect to the randomness of the trajectory $\tau$ (state transitions, observations, and any stochasticity in $\pi$) from timestep $t$ onward, conditioned on the joint history $(h_t^c, h_t^u)$ and action $a_t$. $Q^\pi(h_t^c, h_t^u, a_t)$ represents the expected discounted return of the team when the joint history $(h_t^c, h_t^u)$ has been observed and the controlled agents take action $a_t$.

The \textbf{controlled history action-value function} is defined as:
\begin{equation}
    Q^\pi(h_t^c, a_t) = \mathbb{E}_{\tau \sim \pi,\, h_t^u \mid h_t^c, a_t}\!\left[\sum_{k=0}^{T-t} \gamma^k\, r_{t+k}\right]
\end{equation}
where the expectation additionally averages over the uncontrolled history $h_t^u$ given $h_t^c$. The two levels are related by:
\begin{equation}
    Q^\pi(h_t^c, a_t) = \mathbb{E}_{h_t^u \mid h_t^c}\!\big[Q^\pi(h_t^c, h_t^u, a_t)\big]
\end{equation}

The joint history and controlled history \textbf{value functions} are defined as:
\begin{equation}
    V^\pi(h_t^c, h_t^u) = \mathbb{E}_{\tau \sim \pi \mid h_t^c, h_t^u}\!\left[\sum_{k=0}^{T-t} \gamma^k\, r_{t+k}\right], \qquad V^\pi(h_t^c) = \mathbb{E}_{\tau \sim \pi,\, h_t^u \mid h_t^c}\!\left[\sum_{k=0}^{T-t} \gamma^k\, r_{t+k}\right]
\end{equation}
Equivalently, these are recovered from the action-value functions by averaging over actions:
\begin{equation}
    V^\pi(h_t^c, h_t^u) = \mathbb{E}_{a_t \sim \pi(\cdot \mid h_t^c)}\!\big[Q^\pi(h_t^c, h_t^u, a_t)\big],
    \qquad V^\pi(h_t^c) = \mathbb{E}_{a_t \sim \pi(\cdot \mid h_t^c)}\!\big[Q^\pi(h_t^c, a_t)\big]
\end{equation}
And they are related by:
\begin{equation}
    V^\pi(h_t^c) = \mathbb{E}_{h_t^u \mid h_t^c}\!\big[V^\pi(h_t^c, h_t^u)\big]
\end{equation}

Finally, we define the joint and controlled history \textbf{advantage functions}:
\begin{equation}
    A^\pi(h_t^c, h_t^u, a_t) = Q^\pi(h_t^c, h_t^u, a_t) - V^\pi(h_t^c, h_t^u),
    \qquad A^\pi(h_t^c, a_t) = Q^\pi(h_t^c, a_t) - V^\pi(h_t^c)
\end{equation}
Note again that:
\begin{equation}
    A^\pi(h_t^c, a_t) = \mathbb{E}_{h_t^u \mid h_t^c}\!\big[A^\pi(h_t^c, h_t^u, a_t)\big]
\end{equation}

\subsection{Proximal Policy Optimization}

Controlled policies are trained with the Proximal Policy Optimization (PPO) algorithm \citep{schulman2017ppo}. Recall that since controlled policies are conditioned only controlled histories, the PPO actor loss is defined as:
\begin{equation}
    \mathcal{L}^{PPO}(\theta) = -\mathbb{E}_{h^c, a \sim \pi_{\theta_{\text{old}}}}\!\left[\min\!\Big(\rho(\theta)\, A(h^c, a),\; \operatorname{clip}\big(\rho(\theta),\, 1 - \epsilon,\, 1 + \epsilon\big)\, A(h^c, a)\Big)\right]
\end{equation}
where $\rho(\theta) = \frac{\pi_\theta(a \mid h^c)}{\pi_{\theta_{\text{old}}}(a \mid h^c)}$ is the importance sampling ratio between the current policy $\pi_\theta$ and the behavior policy $\pi_{\theta_{\text{old}}}$ used to collect the trajectory. Note that the true loss we wish to estimate has no dependence on uncontrolled histories: they have been marginalized out. We care only about choosing the best action given a controlled history. This paper will compare various estimators for $A(h^c, a)$ and demonstrate that estimators utilizing joint history critics (taking into account uncontrolled histories) can be more accurate than those using only controlled histories.

\subsection{Generalized Advantage Estimation}
The advantage $A(h^c, a)$ in the PPO objective is not directly observed and must be estimated from rollouts. A commonly used technique (and the technique used for all experiments in this paper) is Generalized Advantage Estimation (GAE) \citep{schulman2016gae}, which trades off bias and variance through an exponentially-weighted average of $k$-step TD residuals. Given a learned value function $V_\phi$ and a trajectory $\tau$, the one-step TD residual at timestep $t$ is
\begin{equation}
    \delta_t = r_t + \gamma V_\phi(h^c_{t+1}) - V_\phi(h^c_t),
\end{equation}
and the GAE advantage is defined as
\begin{equation}
    \hat{A}^{\text{GAE}(\gamma, \lambda)}_t = \sum_{l=0}^{\infty} (\gamma \lambda)^l \, \delta_{t+l},
\end{equation}
where $\gamma \in [0, 1]$ is the discount factor and $\lambda \in [0, 1]$ controls the bias-variance tradeoff. Setting $\lambda = 0$ recovers the one-step TD estimator $\hat{A}_t = \delta_t$, which has low variance but is biased by errors in $V_\phi$. Setting $\lambda = 1$ recovers the Monte Carlo advantage $\hat{A}_t = \sum_{l=0}^{\infty} \gamma^l r_{t+l} - V_\phi(h^c_t)$, which is unbiased (with respect to the true value function) but suffers from high variance. Intermediate values of $\lambda$ interpolate between these regimes.

Note also that we can recover an estimate of the action-value function from the GAE advantage by adding back the value baseline:
\begin{equation}
    \hat{Q}^{\text{GAE}(\gamma, \lambda)}_t = \hat{A}^{\text{GAE}(\gamma, \lambda)}_t + V_\phi(h^c_t).
\end{equation}

In the standard single-agent setting, the value function $V_\phi$ is conditioned on the same information as the policy. In our multi-agent setting, however, we have a choice: $V_\phi$ may condition on the controlled history $h^c$ alone, or on the joint history $(h^c, h^u)$ that additionally includes uncontrolled agents' observations and actions. Section~\ref{sec:theory} compares these two choices and analyzes how the conditioning of the critic affects the variance of the resulting advantage estimator.

\paragraph{Training and Evaluation in Practice}
Similar to the training and evaluation procedure specified by \citet{wang2024naht}, we assume access to a training and test set of teammates, $X_{\text{train}}$ and $X_{\text{test}}$, which are drawn from the same distribution over teammate types. NAHT algorithms learn from interaction with teammates within the training set $X_{\text{train}}$, while evaluations are conducted with teammates in the test set, $X_{\text{test}}$.

\section{Theoretical Analysis}
\label{sec:theory}

Existing algorithms for NAHT are policy-gradient based approaches, implemented via PPO, so our theoretical analysis examines the bias and variance of advantage estimators. We will show that both estimators are unbiased for the controlled advantage function, and then show that the joint estimator is often lower variance than the controlled estimator.

The following assumptions and definitions are used throughout this section:
\begin{assumption}[Unbiased value function estimators]
\label{ass:unbiased-value}
Let $\hat{V}(h_t^c, h_t^u)$ and $\hat{V}(h_t^c)$ be any unbiased estimators of the joint and controlled value functions, respectively, i.e., any random variables satisfying:
\begin{equation}
    \mathbb{E}_{\tau \sim \pi \mid h_t^c, h_t^u}\!\big[\hat{V}(h_t^c, h_t^u)\big] = V(h_t^c, h_t^u),
    \qquad
    \mathbb{E}_{h_t^u,\, \tau \sim \pi \mid h_t^c}\!\big[\hat{V}(h_t^c)\big] = V(h_t^c)
\end{equation}
This is an idealization standard in the theoretical MARL literature \citep{lyu2023centralized}; in practice it holds approximately for sufficiently expressive critics trained to convergence.
\end{assumption}

\begin{assumption}[Unbiased action-value function estimator]
\label{ass:unbiased-action-value}
Let $\hat{Q}(h_t^c, a_t)$ be any unbiased estimator of the controlled action-value function $Q(h_t^c, a_t)$ i.e., any random variable satisfying:
\begin{equation}
    \mathbb{E}_{h_t^u,\, \tau \sim \pi \mid h_t^c, a_t}\!\big[\hat{Q}(h_t^c, a_t)\big] = Q(h_t^c, a_t)
\end{equation}
This is satisfied with a monte-carlo returns estimator, or a GAE estimator bootstrapping with a value function satisfying Assumption~\ref{ass:unbiased-value}.
\end{assumption}

\begin{definition}[Controlled and joint advantage estimators]
Let $\hat{Q}_t^{(1)}$ and $\hat{Q}_t^{(2)}$ be any two (possibly distinct) estimators satisfying Assumption~\ref{ass:unbiased-action-value}. Note that in practice, $\hat{Q}_t^{(1)}$ will likely be constructed with access to a controlled critic, and $\hat{Q}_t^{(2)}$ will likely be constructed with access to a joint critic.
\begin{equation}
    \hat{A}_{controlled}(a_t, h_t^c) = \hat{Q}_t^{(1)} - \hat{V}(h_t^c), \qquad \hat{A}_{joint}(a_t, h_t^c, h_t^u) = \hat{Q}_t^{(2)} - \hat{V}(h_t^c, h_t^u)
\end{equation}
\end{definition}

\subsection{Bias of Controlled and Joint Estimators}
We first establish conditions under which both controlled and joint advantage estimators are unbiased for the controlled advantage that we wish to estimate as a term of the PPO loss.

\begin{theorem}[Unbiasedness of Advantage Estimators]
\label{theorem:unbiased}
\begin{equation}
    \mathbb{E}_{h_t^u,\, \tau \sim \pi \mid h_t^c, a_t}\!\big[\hat{A}_{controlled}\big] = \mathbb{E}_{h_t^u,\, \tau \sim \pi \mid h_t^c, a_t}\!\big[\hat{A}_{joint}\big] = A(h_t^c, a_t)
\end{equation}
\end{theorem}

The proof is provided in Appendix~\ref{app:proof-unbiased}.
Since both joint and controlled estimators are unbiased, they appear to each be reasonable choices to use in practice.

\subsection{Variance of Controlled and Joint Estimators}

When we examine the variances of joint and controlled estimators, we find that the joint estimator is often preferable.

\begin{definition}
Let $q^{(2)}(h_t^c, h_t^u, a_t) := \mathbb{E}[\hat{Q}_t^{(2)} \mid h_t^c, h_t^u, a_t]$ denote the conditional expectation of the joint-history action-value estimator given the full joint history and action.
\end{definition}

\begin{assumption}[Deterministic value function estimators]
\label{ass:deterministic-value}
The value function estimators $\hat{V}(h_t^c, h_t^u)$ and $\hat{V}(h_t^c)$ are deterministic functions of their inputs. That is, conditioned on $(h_t^c, h_t^u)$ (resp. $h_t^c$), they take a single fixed value with no remaining randomness:
\begin{equation}
    \operatorname{Var}_{\tau \sim \pi \mid h_t^c, h_t^u, a_t}\!\big[\hat{V}(h_t^c, h_t^u)\big] = 0,
    \qquad
    \operatorname{Var}_{h_t^u,\, \tau \sim \pi \mid h_t^c}\!\big[\hat{V}(h_t^c) \big] = 0
\end{equation}
This holds for any critics which are deterministic neural networks, so their outputs vary only across their inputs ($h_t^c, h_t^u$ or $h_t^c$), not across the trajectory rollout randomness conditional on those inputs.
\end{assumption}

We decompose the variance difference between a joint and controlled estimator for $A(h_t^c, a_t)$. Note that the joint estimator has lower variance (and we achieve a ``variance reduction'') when the difference is \textit{positive}.

\begin{theorem}[Advantage Variance Decomposition]
\label{theorem:decomp}
For any fixed $(h_t^c, a_t)$:
\begin{align}
    &\operatorname{Var}_{h_t^u,\, \tau \sim \pi \mid h_t^c, a_t}\!\big[\hat{A}_{controlled}\big] - \operatorname{Var}_{h_t^u,\, \tau \sim \pi \mid h_t^c, a_t}\!\big[\hat{A}_{joint}\big] \nonumber \\
    &\quad = \operatorname{Var}_{h_t^u,\, \tau \sim \pi \mid h_t^c, a_t}\!\big[\hat{Q}_t^{(1)}\big] - \operatorname{Var}_{h_t^u,\, \tau \sim \pi \mid h_t^c, a_t}\!\big[\hat{Q}_t^{(2)}\big] \nonumber \\
    &\qquad + \operatorname{Var}_{h_t^u \mid h_t^c, a_t}\!\big[q^{(2)}(h_t^c, h_t^u, a_t)\big] - \operatorname{Var}_{h_t^u \mid h_t^c, a_t}\!\big[q^{(2)}(h_t^c, h_t^u, a_t) - \hat{V}(h_t^c, h_t^u)\big]
\end{align}
\end{theorem}
The variance difference decomposes into two contributions:
\begin{itemize}
    \item $\operatorname{Var}[\hat{Q}_t^{(1)}] - \operatorname{Var}[\hat{Q}_t^{(2)}]$ captures the difference in raw variance between the two action-value estimators. For example, if both estimators are the monte-carlo estimator, then this term is $0$. If both estimators are deep-Q networks, or GAE estimators, then it is more difficult to make rigorous claims about the sign of this term. Note that this term is ``out of control" for the critic, which just exists to subtract a baseline from the action-value estimator. As a result, and because we have estimators for which the term is surely $0$, we focus on the second term for our analysis.
    \item The last two terms, $\operatorname{Var}[q^{(2)}] - \operatorname{Var}[q^{(2)} -\hat{V}(h_t^c, h_t^u)]$, measure how much of the $h_t^u$-variation in $q^{(2)}$ is actually cancelled by subtracting the joint baseline. These terms capture the variance reduction benefit of having a baseline that co-varies with the return across uncontrolled histories. We define this to be the \textbf{baseline-cancellation term}.
\end{itemize}
The proof is provided in Appendix~\ref{app:proof-decomp}.

\begin{assumption}
\label{ass:special-case}
The special case where $\hat{Q}_t^{(2)}$ additionally is an unbiased estimator for the \textit{joint} action-value function:
\begin{equation}
    \mathbb{E}[\hat{Q}_t^{(2)} \mid h_t^c, h_t^u, a_t] = Q(h_t^c, h_t^u, a_t)
\end{equation}
This occurs when $\hat{Q}_t^{(2)}$ is GAE bootstrapped through the  joint history critic $\hat{V}(h_t^c, h_t^u)$.
\end{assumption}

\begin{corollary}[Variance Decomposition Simplification]
Under Assumption~\ref{ass:special-case} and the exact-critic idealization $\hat{V}(h_t^c, h_t^u)=V(h_t^c, h_t^u)$, Theorem~\ref{theorem:decomp} reduces to:
\begin{align}
    &\operatorname{Var}[\hat{A}_{controlled}] - \operatorname{Var}[\hat{A}_{joint}] \nonumber \\
    &\quad = \operatorname{Var}_{h_t^u,\, \tau \sim \pi \mid h_t^c, a_t}[\hat{Q}_t^{(1)}] - \operatorname{Var}_{h_t^u,\, \tau \sim \pi \mid h_t^c, a_t}[\hat{Q}_t^{(2)}]
    \nonumber \\
    &\qquad + \operatorname{Var}_{h_t^u \mid h_t^c, a_t}[Q(h_t^c, h_t^u, a_t)] - \operatorname{Var}_{h_t^u \mid h_t^c, a_t}[A(h_t^c, h_t^u, a_t)]
\end{align}
\end{corollary}

\subsection{When Is the Baseline-Cancellation Term Large?}
\label{sec:baseline-cancellation}

We now consider conditions under which the baseline-cancellation term, $\operatorname{Var}[q^{(2)}] - \operatorname{Var}[q^{(2)} - V(h_t^c, h_t^u)]$, is positive, creating a variance reduction when the same action-value estimator is used for each type of critic. For clarity, we restrict attention to the case where $\hat{Q}_t^{(2)}$ satisfies Assumption~\ref{ass:special-case}. In this case, the baseline-cancellation term reduces to:
\begin{equation}
    \operatorname{Var}_{h_t^u \mid h_t^c, a_t}\!\big[Q(h_t^c, h_t^u, a_t)\big] - \operatorname{Var}_{h_t^u \mid h_t^c, a_t}\!\big[A(h_t^c, h_t^u, a_t)\big]
\end{equation}
Using the identity $\operatorname{Var}[Q] - \operatorname{Var}[Q - V] = 2\operatorname{Cov}[Q, V] - \operatorname{Var}[V]$, the baseline-cancellation term is positive if and only if:
\begin{equation}
    \operatorname{Cov}_{h_t^u \mid h_t^c, a_t}\!\big[Q(h_t^c, h_t^u, a_t),\, V(h_t^c, h_t^u)\big] > \frac{1}{2}\operatorname{Var}_{h_t^u \mid h_t^c, a_t}\!\big[V(h_t^c, h_t^u)\big]
\end{equation}
Writing $\sigma_Q = \sqrt{\operatorname{Var}_{h_t^u \mid h_t^c, a_t}[Q(h_t^c, h_t^u, a_t)]}$, $\sigma_V = \sqrt{\operatorname{Var}_{h_t^u \mid h_t^c, a_t}[V(h_t^c, h_t^u)]}$, and $\rho = \operatorname{Corr}_{h_t^u \mid h_t^c, a_t}[Q(h_t^c, h_t^u, a_t),\, V(h_t^c, h_t^u)]$, this becomes:
\begin{equation}
    \rho > \frac{\sigma_V}{2\sigma_Q}
\end{equation}
Intuitively, the joint baseline $V(h_t^c, h_t^u)$ reduces variance by canceling variation in $Q(h_t^c, h_t^u, a_t)$ that is driven by uncontrolled histories rather than by the action itself. For this cancellation to be effective, two conditions must hold:

\paragraph{1. $V$ and $Q$ must move together across uncontrolled histories.} 
\label{para:insight1}
If uncontrolled histories that are generally good for the current policy (high $V$) are also good for the specific action taken (high $Q$), then subtracting $V$ strips out shared variation. High correlation $\rho$ captures this: if $Q$ and $V$ move in lockstep across $h_t^u$, then $Q - V$ is nearly constant in $h_t^u$ and the cancellation is near-perfect. \textbf{Note that if $a_t$ is chosen with probability close to 1} (the policy is nearly deterministic), then $V \approx Q(a_t)$, and the correlation is certain. In contrast, if $a_t$ is chosen with low probability by the current policy, then a correlation (and variance reduction) is unlikely.
\paragraph{2. $V$ must not swing more than $Q$.} Even with perfect correlation, if the baseline varies more than the quantity it is trying to cancel ($\sigma_V > 2\sigma_Q$), subtracting it introduces more variance than it removes. Putting these together, the baseline-cancellation term is most positive when $Q$ and $V$ are highly correlated \emph{and} have comparable magnitudes of variation across uncontrolled histories.

\subsection{From Per-Pair Advantage Variance to Visitation Weighted Advantage Variance}

Our variance analysis thus far fixes a specific state-action pair $(h_t^c, a_t)$ and compares the conditional variance of the two advantage estimators at that pair. For some $(h_t^c, a_t)$, the joint baseline will reduce variance; for others, the condition $\rho > \sigma_V / 2\sigma_Q$ will fail and the joint baseline will increase variance. PPO, however, is intuitively influenced by \emph{aggregate} gradient variance, which is a visitation-weighted average:
\begin{equation}
    \mathbb{E}_{(h_t^c, a_t) \sim \pi_{\theta_{\text{old}}}}\!\Big[\operatorname{Var}\!\big[\hat{A}(h_t^c, a_t)\big]\Big]
\end{equation}
The relevant question is therefore whether this expected variance is lower under $\hat{A}_{joint}$ than under $\hat{A}_{controlled}$.

\paragraph{Intuitive argument.} We conjecture that the aggregate gradient variance is typically lower under the joint critic, i.e.,
\begin{equation}
    \mathbb{E}_{(h_t^c, a_t) \sim \pi_{\theta_{\text{old}}}}\!\Big[\operatorname{Var}\!\big[\hat{A}_{controlled}\big] - \operatorname{Var}\!\big[\hat{A}_{joint}\big]\Big] > 0
\end{equation}
The intuition is that the state-action pairs where the joint critic most effectively reduces variance are precisely the pairs the current policy visits most often (Section~\ref{sec:baseline-cancellation}). We refer to this intuition to explain the empirical greater performance of joint history critics for multiple tasks showcased in Section~\ref{sec:experiments}.

\section{Illustrative Example: A Coordination Game with Variance Reduction}
\label{sec:coordination-example}
To build intuition for when the joint baseline meaningfully reduces variance, we construct a simple coordination game where the joint history critic produces a substantial --- though not complete --- variance reduction over the controlled critic.

\paragraph{Setup.} We consider a game with:
\begin{itemize}
    \item \textbf{Agents:} 2 controlled agents, 4 uncontrolled agents.
    \item \textbf{Action space:} $\{0, 1\}^6$.
    \item \textbf{Uncontrolled policy:} At the start of each episode, the uncontrolled team samples a shared policy uniformly from $\{\text{Mostly 0}, \text{Mostly 1}\}$. Under \emph{Mostly 0}, each uncontrolled agent independently outputs $0$ with probability $0.9$ and $1$ with probability $0.1$. Under \emph{Mostly 1}, each uncontrolled agent independently outputs $0$ with probability $0.25$ and $1$ with probability $0.75$.
    \item \textbf{Observation space:} Each agent observes only its own past actions. No additional observations are available, and the optimal policy does not condition on any observation.
    \item \textbf{Reward:} $r(\mathbf{a}) = 1$ if all six agents output the same value (either all zeros or all ones), and $r(\mathbf{a}) = 0$ otherwise.
    \item \textbf{Horizon and discount:} Episode length $T$, discount $\gamma = 1$.
\end{itemize}

We analyze a controlled policy that plays $\{0, 0\}$ with probability $0.8$ and $\{1, 1\}$ with probability $0.2$ at every history, and we focus on the advantage of the action $a_t = \{0,0\}$. We take $\hat{Q}_t$ to be a single-step TD estimator that bootstraps through the appropriate critic:
\begin{equation}
    \hat{Q}_t^{(1)} = r_t + V(h_{t+1}^c), \qquad \hat{Q}_t^{(2)} = r_t + V(h_{t+1}^c, h_{t+1}^u)
\end{equation}

\paragraph{Value functions.} Let $p_0 = 0.8$ and $p_1 = 0.2$ denote the controlled policy's probabilities of playing $\{0,0\}$ and $\{1,1\}$, respectively. Under each uncontrolled policy, we compute the expected per-step reward:
\begin{align}
    \mathbb{E}[r \mid \text{Mostly 0}] &= p_0 \cdot 0.9^4 + p_1 \cdot 0.1^4 = 0.8(0.6561) + 0.2(0.0001) = 0.5249 \\
    \mathbb{E}[r \mid \text{Mostly 1}] &= p_0 \cdot 0.25^4 + p_1 \cdot 0.75^4 = 0.8(0.00391) + 0.2(0.3164) = 0.0665
\end{align}
Marginalizing over the uncontrolled policy, $\mathbb{E}[r] = 0.5(0.5249) + 0.5(0.0665) = 0.2957$. The value functions at timestep $t$ (with $K = T - t$ steps remaining) are these per-step expectations scaled by $K$:
\begin{equation}
    V(h_t^c) = K \cdot 0.2957, \qquad V(h_t^c, h_t^{M_0}) = K \cdot 0.5249, \qquad V(h_t^c, h_t^{M_1}) = K \cdot 0.0665
\end{equation}
where $h_t^{M_0}$ and $h_t^{M_1}$ denote uncontrolled histories consistent with Mostly 0 and Mostly 1, respectively. We also compute the controlled action-value for $a_t = \{0,0\}$. At the current step, the expected reward given $a_t = \{0,0\}$ depends on the uncontrolled policy:
\begin{equation}
    \begin{aligned}
    \mathbb{E}[r \mid a_t = \{0,0\}, \text{Mostly 0}] &= 0.9^4 = 0.6561, \\
    \mathbb{E}[r \mid a_t = \{0,0\}, \text{Mostly 1}] &= 0.25^4 \approx 0.00391
    \end{aligned}
\end{equation}
Marginalizing: $\mathbb{E}[r \mid a_t = \{0,0\}] = 0.5(0.6561) + 0.5(0.00391) \approx 0.330$. Adding the expected value of the remaining $K - 1$ steps:
\begin{equation}
    Q(h_t^c, \{0,0\}) = 0.330 + (K-1) \cdot 0.2957
\end{equation}
The true controlled advantage is therefore:
\begin{equation}
    A(h_t^c, \{0,0\}) = Q(h_t^c, \{0,0\}) - V(h_t^c) = 0.330 - 0.2957 = 0.0343
\end{equation}
which is positive, reflecting that $\{0,0\}$ is a better-than-average action under the current policy.

\paragraph{Advantage estimators.} Substituting the value functions:
\begin{equation}
    \hat{A}_{controlled} = r_t + V(h_{t+1}^c) - V(h_t^c) = r_t + (K-1)(0.2957) - K(0.2957) = r_t - 0.2957
\end{equation}
For the joint history estimator, the form depends on the uncontrolled policy in effect:
\begin{align}
    \hat{A}_{joint} \mid \text{Mostly 0} &= r_t + V(h_{t+1}^c, h_{t+1}^{M_0}) - V(h_t^c, h_t^{M_0}) = r_t - 0.5249 \\
    \hat{A}_{joint} \mid \text{Mostly 1} &= r_t + V(h_{t+1}^c, h_{t+1}^{M_1}) - V(h_t^c, h_t^{M_1}) = r_t - 0.0665
\end{align}
Both estimators are unbiased for $A(h_t^c, \{0,0\}) = 0.0343$: for instance, the joint estimator gives $\mathbb{E}[\hat{A}_{joint}] = 0.5(0.6561 - 0.5249) + 0.5(0.00391 - 0.0665) = 0.5(0.1312) + 0.5(-0.0626) = 0.0343$, as required.

\paragraph{Variance computation.} All variances are taken with respect to the joint distribution over the uncontrolled policy $\theta^u \in \{\text{Mostly 0}, \text{Mostly 1}\}$ and the realized reward $r_t \in \{0, 1\}$ given $a_t = \{0, 0\}$. There are four possible outcomes. We first enumerate their probabilities:
\begin{align}
    \Pr(\text{Mostly 0},\, r_t = 1) &= 0.5 \cdot 0.6561 = 0.3281 \\
    \Pr(\text{Mostly 0},\, r_t = 0) &= 0.5 \cdot 0.3439 = 0.1720 \\
    \Pr(\text{Mostly 1},\, r_t = 1) &= 0.5 \cdot 0.00391 = 0.00195 \\
    \Pr(\text{Mostly 1},\, r_t = 0) &= 0.5 \cdot 0.99609 = 0.4980
\end{align}

\emph{Controlled advantage estimator.} Recall $\hat{A}_{controlled} = r_t - 0.2957$, which takes value $0.7043$ when $r_t = 1$ and $-0.2957$ when $r_t = 0$. Its mean is $A(h_t^c, \{0,0\}) = 0.0343$. Computing the variance as expected squared deviation from the mean:
\begin{align}
    \operatorname{Var}\!\big[\hat{A}_{controlled}\big] &= \Pr(r_t = 1) \cdot (0.7043 - 0.0343)^2 + \Pr(r_t = 0) \cdot (-0.2957 - 0.0343)^2 \nonumber \\
    &= 0.330 \cdot (0.6700)^2 + 0.670 \cdot (-0.330)^2 \nonumber \\
    &= 0.330 \cdot 0.4489 + 0.670 \cdot 0.1089 \nonumber \\
    &\approx 0.221
\end{align}

\emph{Joint history advantage estimator.} Recall that $\hat{A}_{joint} = r_t - 0.5249$ under Mostly 0 and $\hat{A}_{joint} = r_t - 0.0665$ under Mostly 1. The four possible values of $\hat{A}_{joint}$ are therefore:
\begin{align}
    \hat{A}_{joint}(\text{Mostly 0}, r_t = 1) &= 1 - 0.5249 = 0.4751 \\
    \hat{A}_{joint}(\text{Mostly 0}, r_t = 0) &= 0 - 0.5249 = -0.5249 \\
    \hat{A}_{joint}(\text{Mostly 1}, r_t = 1) &= 1 - 0.0665 = 0.9335 \\
    \hat{A}_{joint}(\text{Mostly 1}, r_t = 0) &= 0 - 0.0665 = -0.0665
\end{align}
Its mean is $0.0343$. Computing the variance as the expected squared deviation from the mean:
\begin{align}
    \operatorname{Var}\!\big[\hat{A}_{joint}\big] &= 0.3281 \cdot (0.4751 - 0.0343)^2 + 0.1720 \cdot (-0.5249 - 0.0343)^2 \nonumber \\
    &\quad + 0.00195 \cdot (0.9335 - 0.0343)^2 + 0.4980 \cdot (-0.0665 - 0.0343)^2 \nonumber \\
    &= 0.3281 \cdot (0.4408)^2 + 0.1720 \cdot (-0.5592)^2 \nonumber \\
    &\quad + 0.00195 \cdot (0.8992)^2 + 0.4980 \cdot (-0.1008)^2 \nonumber \\
    &= 0.3281 \cdot 0.1943 + 0.1720 \cdot 0.3127 + 0.00195 \cdot 0.8086 + 0.4980 \cdot 0.01016 \nonumber \\
    &\approx 0.124
\end{align}

\paragraph{Result.} The joint history baseline nearly halves the variance of the advantage estimator:
\begin{equation}
    \operatorname{Var}\!\big[\hat{A}_{controlled}\big] \approx 0.221, \qquad \operatorname{Var}\!\big[\hat{A}_{joint}\big] \approx 0.124
\end{equation}
This matches our first insight about variance reduction. For this scenario, the action $\{0,0\}$ receives a high reward when paired with the same types of uncontrolled agents that the policy performs well for: Mostly 0. Action $\{0,0\}$ is also currently chosen with high probability, as we note can be beneficial for variance reduction in insight 1.

\paragraph{Negative Example}
Note that if we change the conditions of this example, we can end up with a scenario where the joint estimator has greater variance than the controlled estimator. Such an example is provided in Appendix~\ref{AppendixA} to help readers gain intuition about the differences between estimators.

\section{Empirical Results}
\label{sec:experiments}

\caroline{An experimental section should follow the following general structure: (1) enumerate research questions and foreshadow results, (2) present methods examined, baselines, metrics, (3) present results for each RQ, section by section.}

This section addresses the primary research question of if and when performance differs between joint and controlled critics in the NAHT problem. We compare performance across a variety of games and previously proposed algorithms for NAHT. Across every (game, algorithm) pair, a joint critic is as effective as a controlled critic, and for many scenarios, a joint critic exhibits significantly better performance.

\subsection{Algorithms}
As a first solution method for NAHT, \citet{wang2024naht} proposed the Policy Optimization with Agent Modeling (POAM) algorithm. POAM deals with a potentially varying number of controlled agents by employing an independent policy gradient approach, where each agent learns an independent policy while performing teammate modeling of all other agents in the team. Each agent's policy conditions on the learned teammate representations, thus enabling the policy to approximate the optimal behavior for collaborating with each type of uncontrolled team. Since we study centralized critics in this paper, we adapt POAM to include centralized observations in the same fashion as \citet{wang2025matnaht}, creating the method Cent Obs POAM. We simply refer to the method as "POAM" throughout the paper.

\citet{wang2025matnaht} proposed MAT-NAHT, a transformer-based method adapted from the Multi-Agent Transformer~\citep{wen2022multi}, in which an encoder-decoder transformer jointly serves as the centralized policy and value network for the controlled subteam. The architecture handles the variable number of controlled agents through attention masking, and the centralized critic is conditioned on the history of observations and actions of all controlled agents. We abbreviate the method to MAT for this paper.

\subsection{Connection to Theory}
All experiments were conducted in an environment that satisfies all assumptions from Section~\ref{sec:theory}. Specifically, controlled critic training runs use GAE estimation and boostrap through the controlled critic while joint critic training runs use GAE estimation bootstrapped through a joint critic. This satisfies Assumption~\ref{ass:unbiased-action-value}, Assumption~\ref{ass:deterministic-value}, and also the stronger condition Assumption~\ref{ass:special-case}. As noted previously Assumption~\ref{ass:unbiased-value} is satisfied when the value networks have converged, and is a standard assumption in previous literature.

\subsection{Coordination Game}
\begin{figure}[h]
  \centering
  \begin{subfigure}[t]{0.48\textwidth}
    \includegraphics[width=\textwidth]{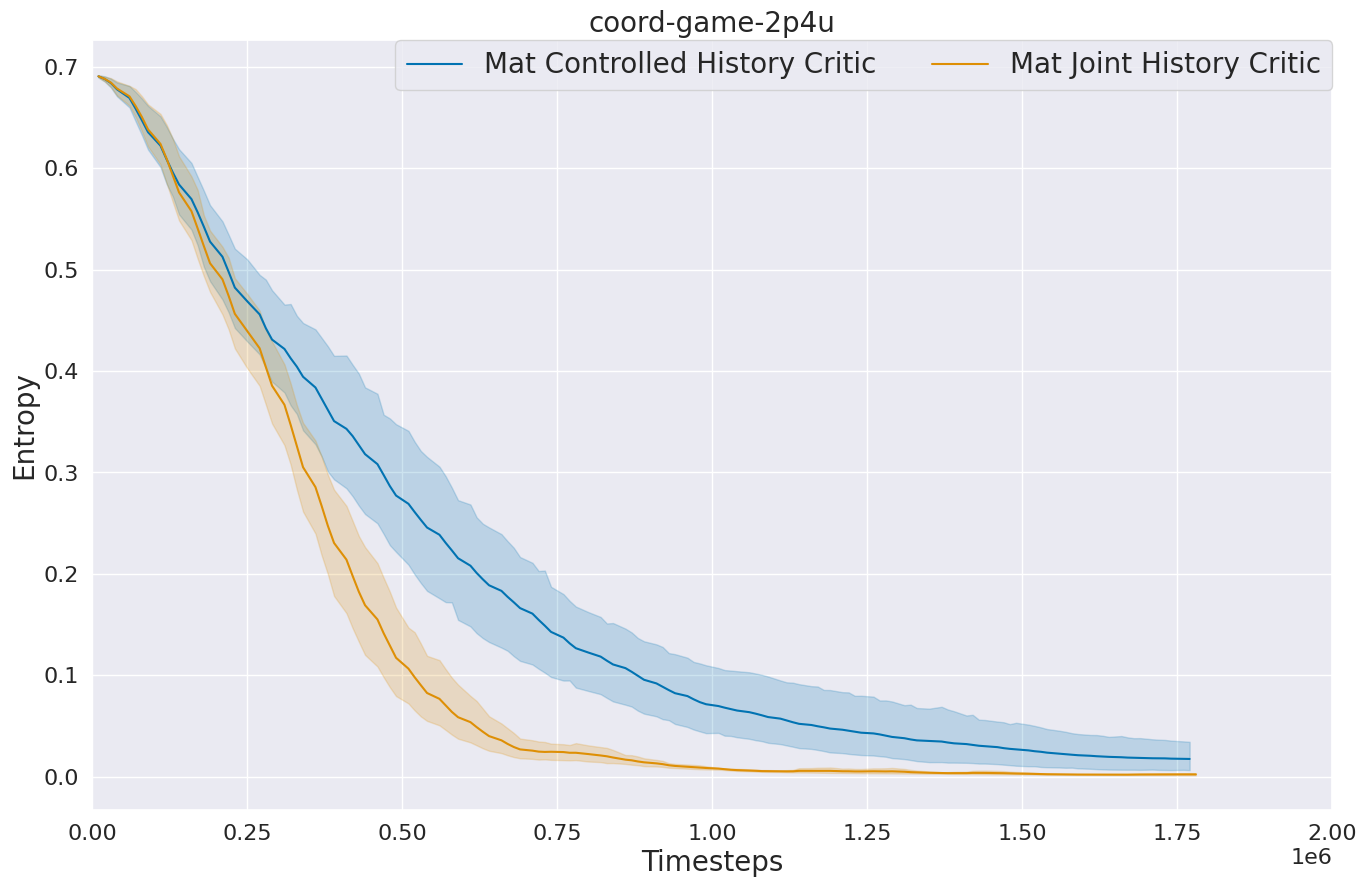}
    \caption{Entropy of MAT policies.}
    \label{fig:coord-mat}
  \end{subfigure}
  \hfill
  \begin{subfigure}[t]{0.48\textwidth}
    \includegraphics[width=\textwidth]{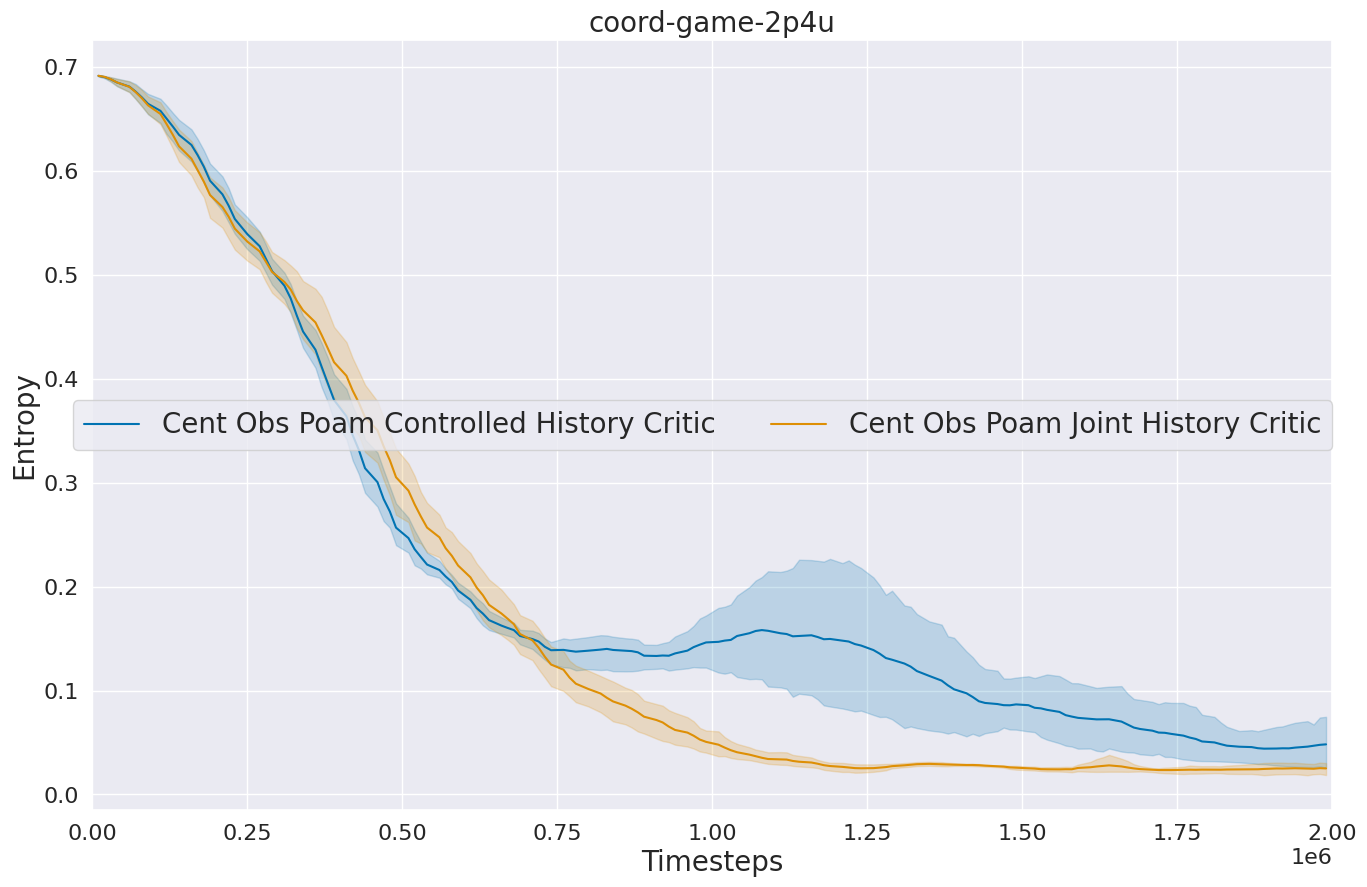}
    \caption{Entropy of POAM policies.}
    \label{fig:coord-poam}
  \end{subfigure}

  \vspace{0.5em}

  \begin{subfigure}[t]{0.48\textwidth}
    \includegraphics[width=\textwidth]{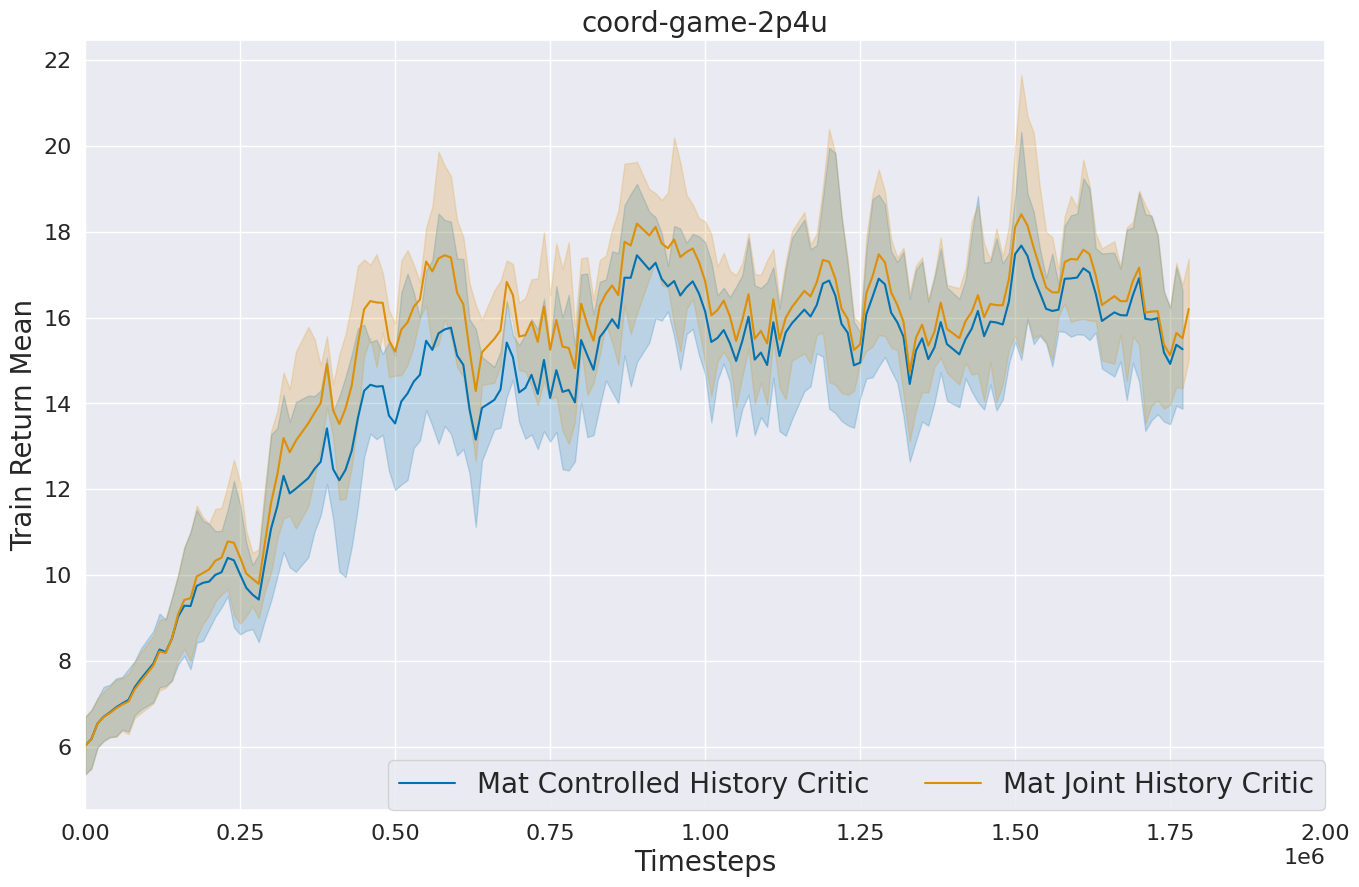}
    \caption{Return curves of MAT policies.}
    \label{fig:return-mat}
  \end{subfigure}
  \hfill
  \begin{subfigure}[t]{0.48\textwidth}
    \includegraphics[width=\textwidth]{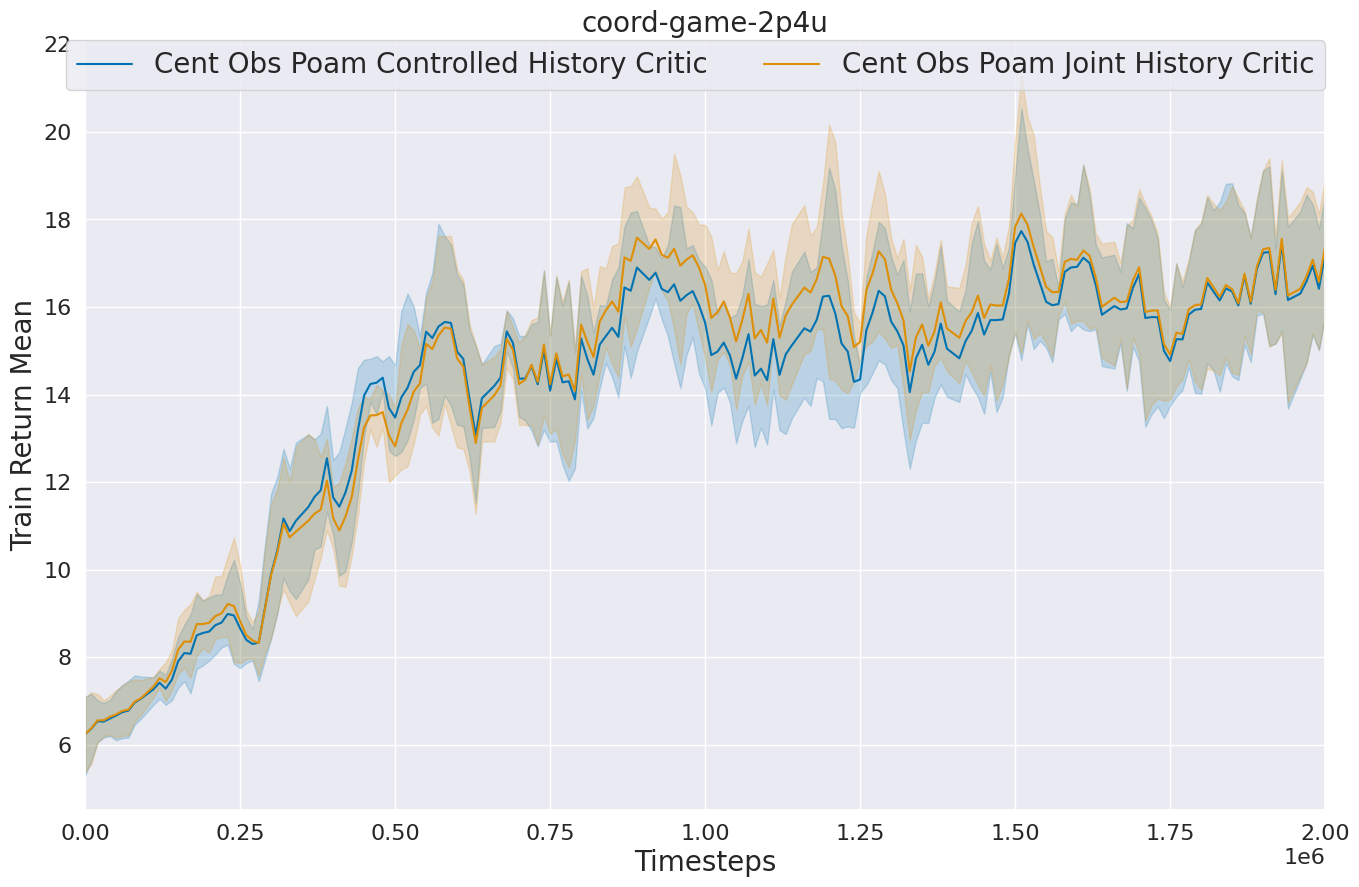}
    \caption{Return curves of POAM policies.}
    \label{fig:return-poam}
  \end{subfigure}

  \caption{Joint history critics have significantly higher sampling efficiency than controlled history critics for coordination game. \caroline{figure text should be larger to be a similar scale to captions.}}
  \label{fig:experiments-grid}
\end{figure}

We first present results for the coordination game described previously. Recall that the optimal policy is to always output $\{0, 0\}$ for every history, so the optimal policy has $0$ entropy. We can see that for the entropy curves of both MAT and POAM, the joint history critic achieves a near-optimal, $0$ entropy policy significantly faster than the controlled history critic. The same story is reflected in the return curves, which demonstrate the greater sampling efficiency of a joint history critic for this task.

\subsection{Stag Hunt}
\caroline{For each game presented, provide a graphic illustrating the game. e.g. for coordination game, you can provide game payoffs as a graphic . For stag hunt, provide a visualization of the gridworld.}
In Stag Hunt \citep{bohmer2020deep}, $N$ predator agents must capture stags and hares in a grid world. Agents and prey can move along the four cardinal directions or stay still, and agents additionally have a \texttt{catch} action. Stags yield a high reward but require coordinated capture: at least two adjacent agents must simultaneously execute the \texttt{catch} action to capture a stag, while a single adjacent agent suffices to capture a hare. If an agent attempts to capture a stag without sufficient teammates adjacent to the same prey, the team is penalized by the miscapture punishment. Upon a successful capture, the participating agents are frozen for the remainder of the episode and removed from the grid, and the captured prey is removed. Agents observe a $(2k+1) \times (2k+1)$ patch centered on their position, while prey movement is governed by a fixed stochastic policy: prey move to a random unoccupied adjacent cell. An episode terminates when all prey are captured, all agents are frozen, or after $T$ timesteps.

\begin{figure}[h]
  \centering
\begin{subfigure}[t]{0.48\textwidth}
    \includegraphics[width=\textwidth]{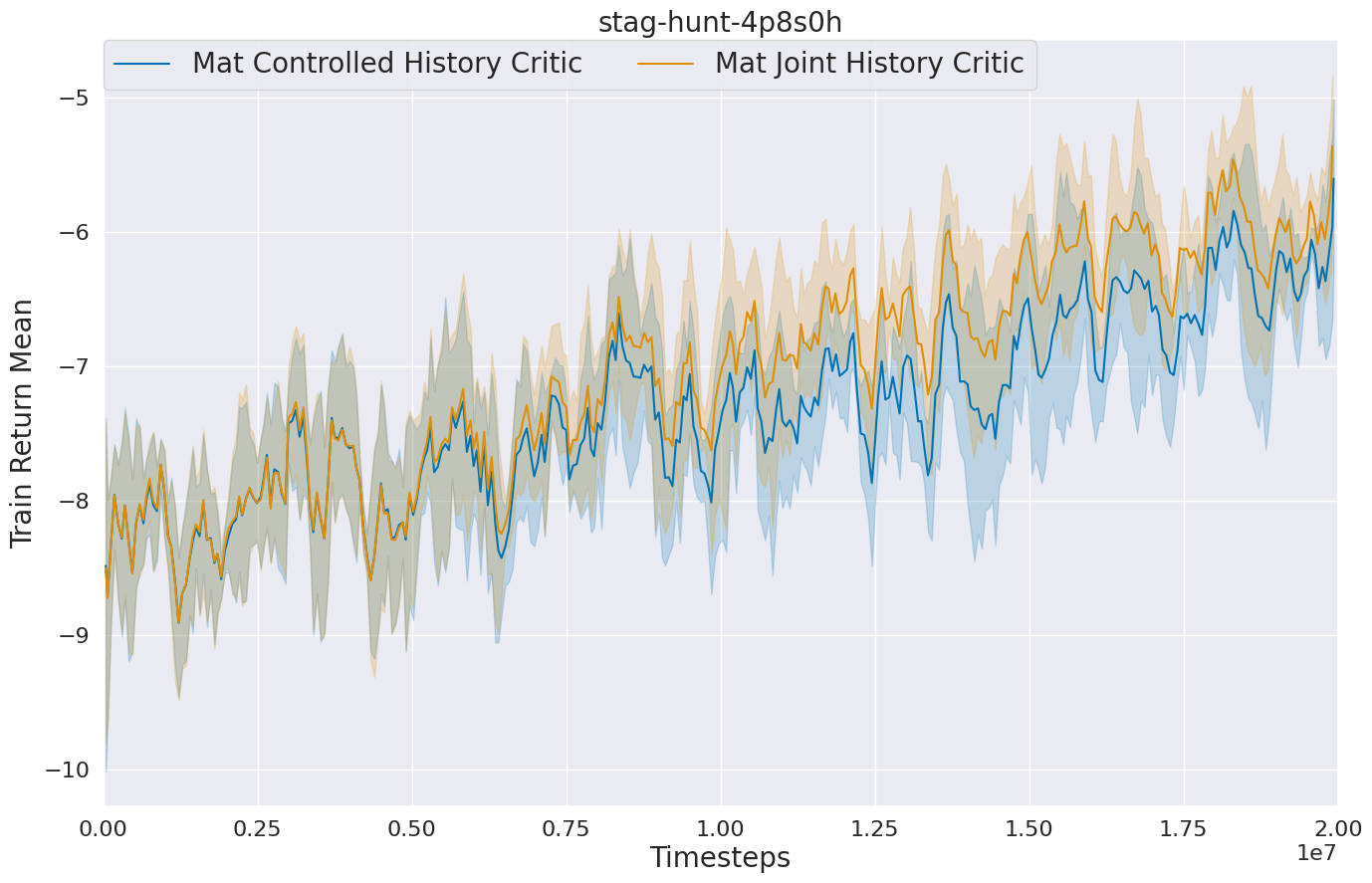}
    \caption{Return curves of MAT policies.}
    \label{fig:stag-return-mat}
  \end{subfigure}
  \hfill
  \begin{subfigure}[t]{0.48\textwidth}
    \includegraphics[width=\textwidth]{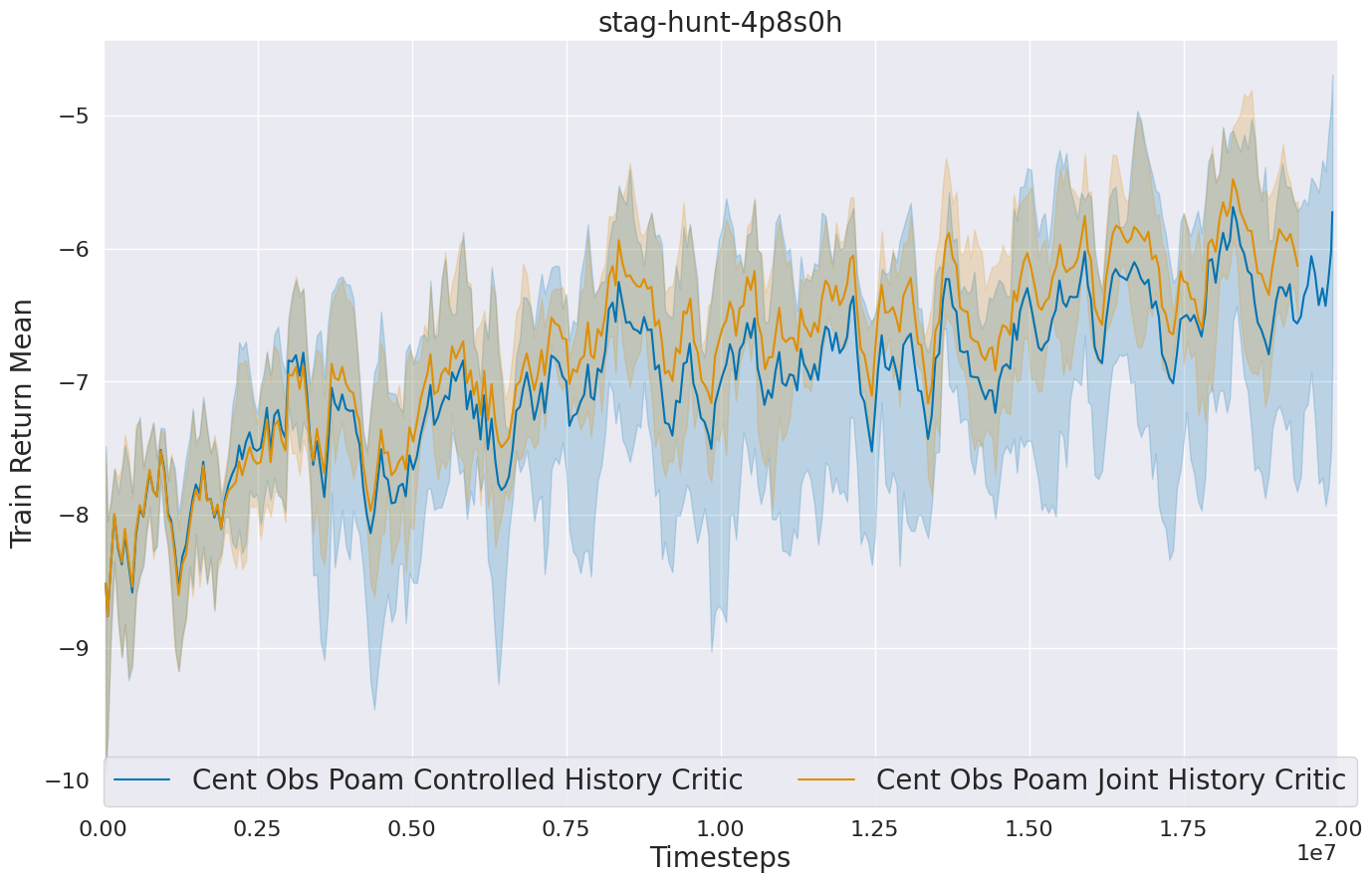}
    \caption{Return curves of POAM policies.}
    \label{fig:stag-return-poam}
  \end{subfigure}
  \hfill
\begin{subfigure}[t]{0.48\textwidth}
    \includegraphics[width=\textwidth]{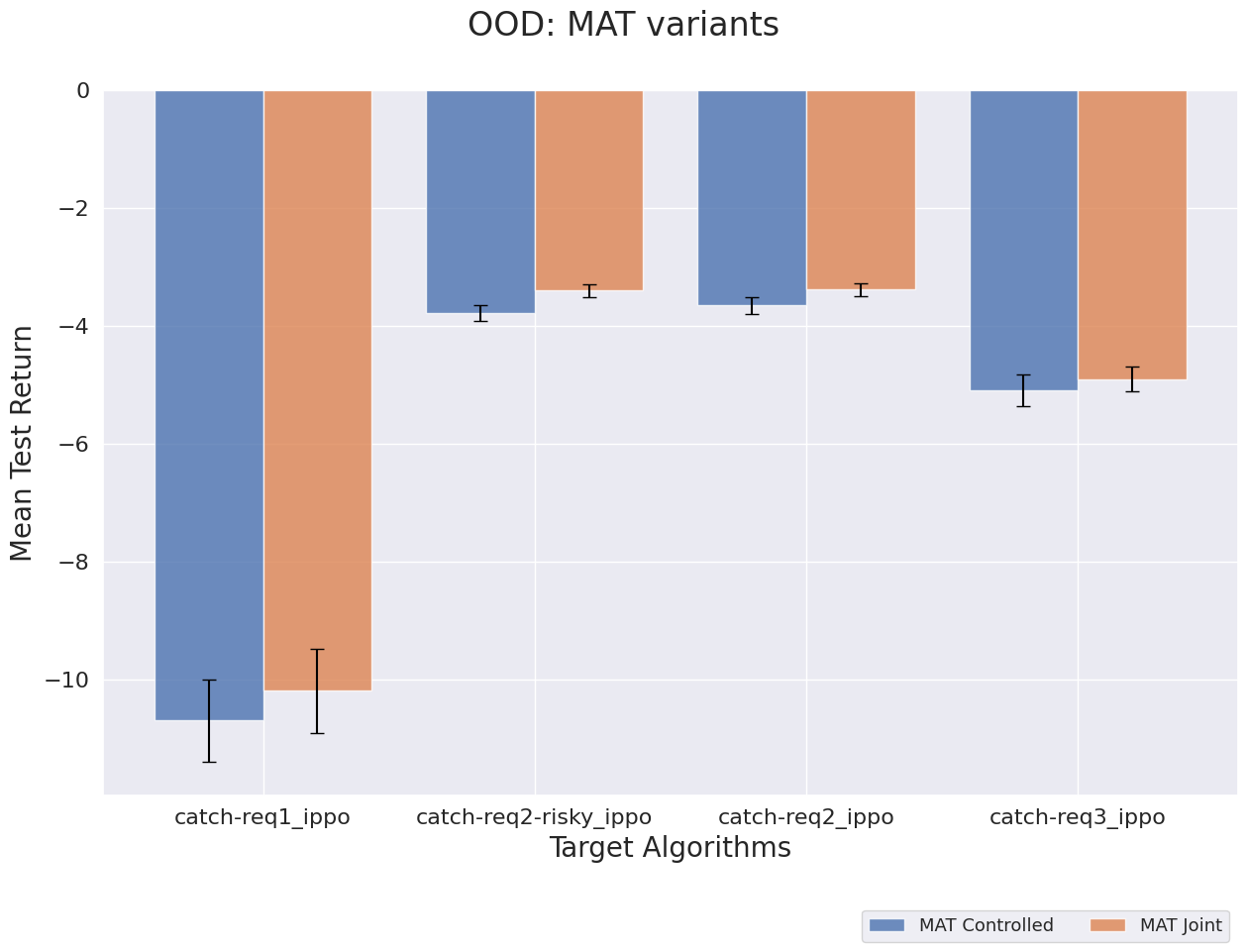}
    \caption{Returns with fixed uncontrolled agent policies (MAT). \caroline{This figure makes it look like there's little difference between controlled and joint history critics. I suggest instead visualizing the difference between the returns of the two methods.}}
    \label{fig:stag-mat}
  \end{subfigure}
  \hfill
  \begin{subfigure}[t]{0.48\textwidth}
    \includegraphics[width=\textwidth]{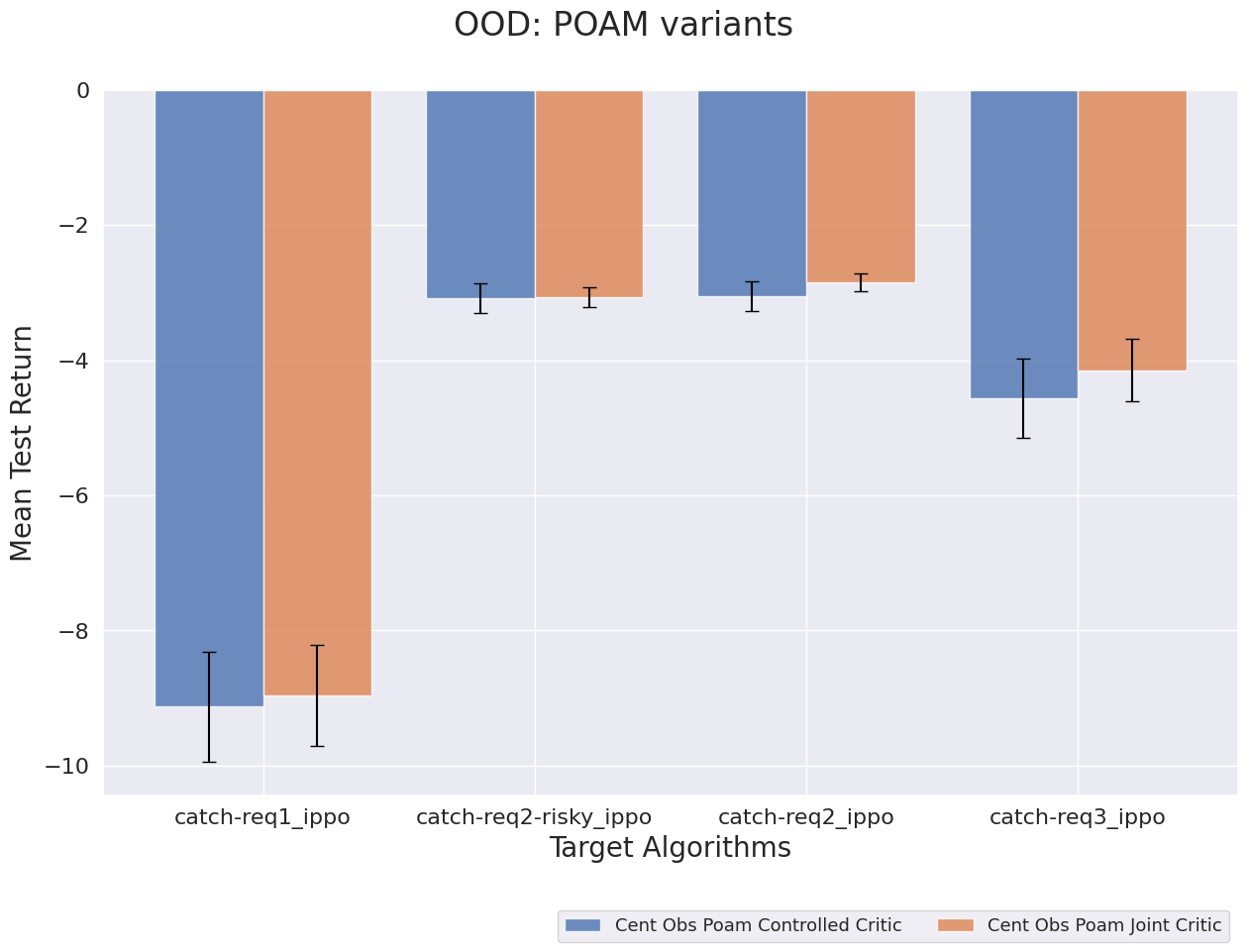}
    \caption{Returns with fixed uncontrolled agent policies (POAM).}
    \label{fig:stag-poam}
  \end{subfigure}
  \caption{Comparison of stag hunt return curves and out of distribution performance across MAT and POAM variants.}
  \label{fig:stag-experiments}
\end{figure}

From the return curves we see that, like coordination game, joint critics offer improved sample efficiency for both MAT and POAM. From the bar charts, we can compare joint and controlled critic policies with individual uncontrolled agents \caroline{the bar plots for stag hunt are against OOD agents, while the learning curves are for in distribution agents. Let's separately analyze the variance reduction for these two cases. I would expect approximation error to be higher against OOD teammates. How would this affect the claimed variance reduction? }. We see that joint critics perform better with every type of uncontrolled agent.

\begin{table}[ht]
\centering
\caption{Stag Hunt MAT Hypothesis Tests.}
\label{tab:paired_results}
\begin{tabular}{lrr}
\hline
Target algorithm & Mean difference (C--J) & Paired $p$-value \\
\hline
catch-req1\_ippo$^{*}$  & -0.501652 & 0.010758\\
catch-req2-risky\_ippo$^{***}$  & -0.376259	 & 0.000388 \\
catch-req2\_ippo$^{*}$  & -0.272097	 & 0.017871 \\
catch-req3\_ippo   & -0.195542	  & 0.426157 \\
\hline
\end{tabular}
\end{table}

We use hypothesis tests across all (controlled, uncontrolled) seed pairs for each comparison to see if the difference in performance between joint and controlled critics is statistically significant. At $\alpha=0.05$, we find statistically significant differences for 3 of the 4 uncontrolled agent types. Note that for the 4th uncontrolled agent type, the mean returns are still greater for a joint critic, and we do not have data to suggest that controlled critics are better for this type of agent. In order to make a statistically significant claim for the superiority of one critic over another, more data would need to be collected.

\subsection{SMAC 8v9}
The SMAC 8v9 challenge \citep{samvelyan2019smac} is a hard-difficulty StarCraft II micromanagement scenario in which $N=8$ allied Marines must defeat a slightly larger enemy team of 9 Marines controlled by the built-in game AI. At each timestep, every allied agent selects from a discrete action set consisting of move (in the four cardinal directions), stop, no-op, or attack-target actions, where the attack actions are restricted to enemy units within the agent's fixed shooting range. Agents observe a partial view of the battlefield limited to a circular sight range, receiving features such as relative positions, health, shield, and unit type of allies and enemies within range; units outside this range are masked from observation. The team receives a shaped reward based on damage dealt to enemies, kills, and a large terminal bonus for winning the engagement, with success requiring coordinated focus-fire and positioning to overcome the numerical disadvantage. An episode terminates when all units on one side are eliminated or after $T$ timesteps, and the scenario is considered solved when the allied team achieves a high win rate against the scripted opponent.

\begin{figure}[h]
  \centering
\begin{subfigure}[t]{0.48\textwidth}
    \includegraphics[width=\textwidth]{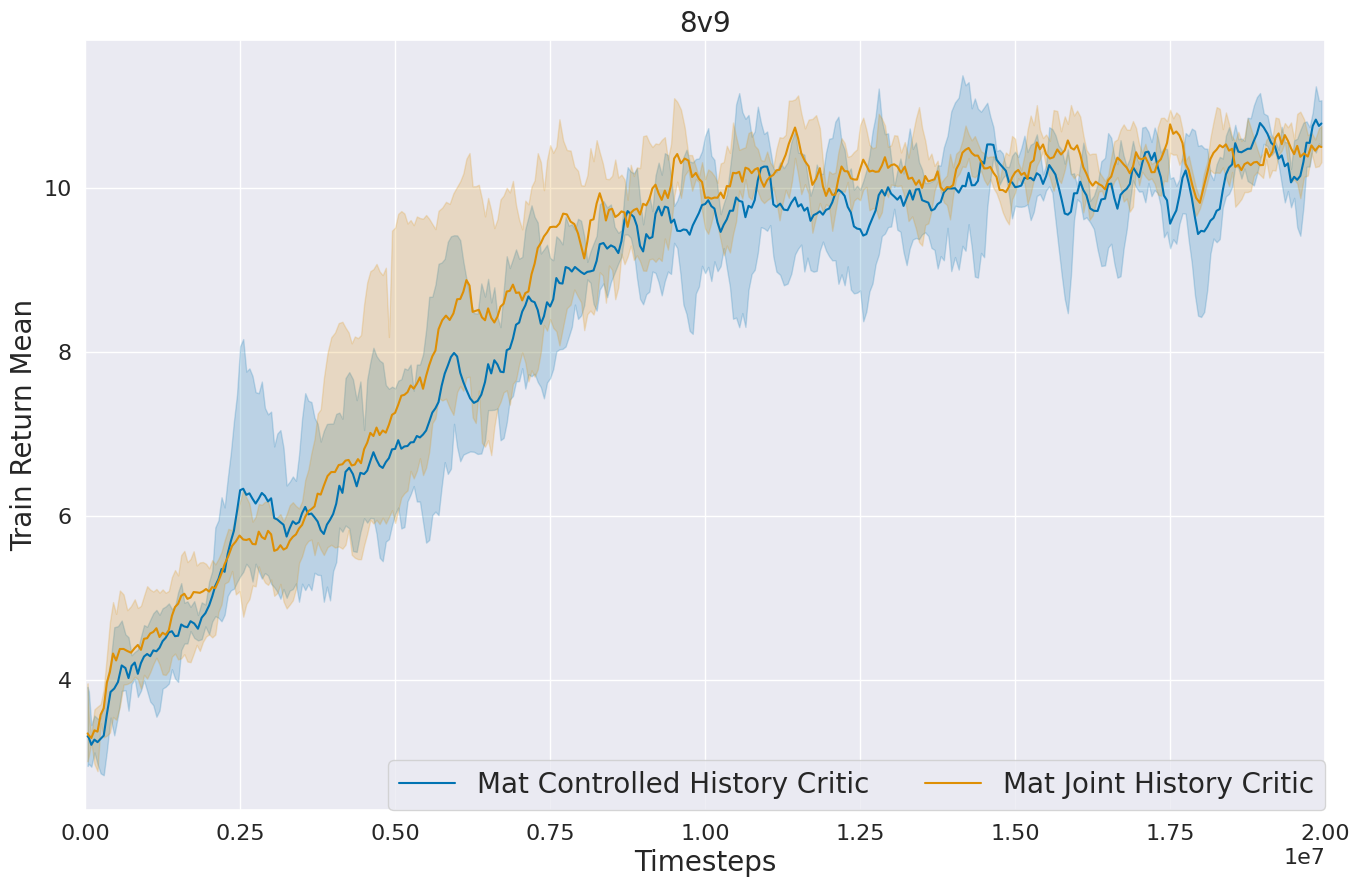}
    \caption{Return curves of MAT policies.}
    \label{fig:smac-return-mat}
  \end{subfigure}
  \hfill
  \begin{subfigure}[t]{0.48\textwidth}
    \includegraphics[width=\textwidth]{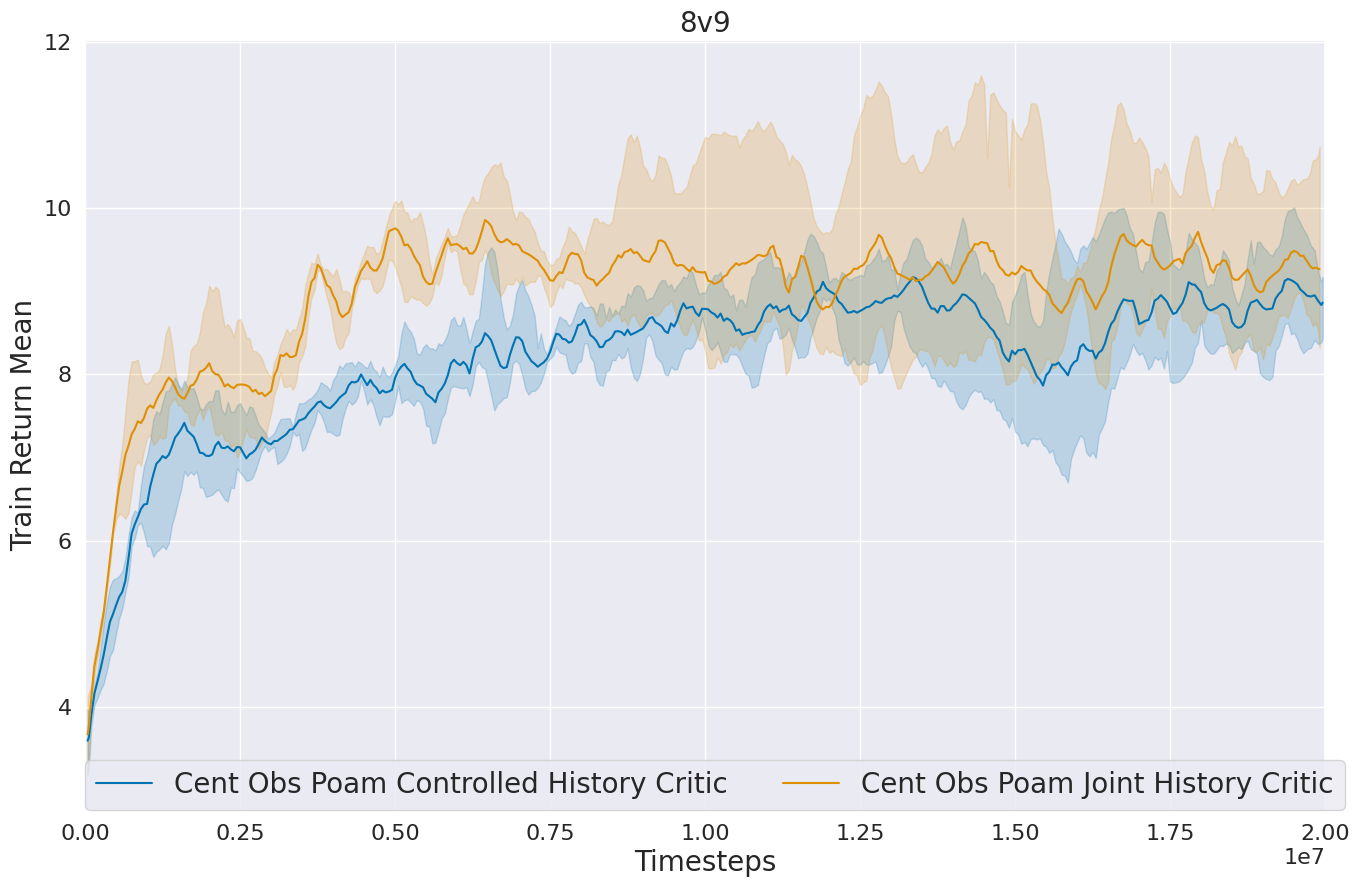}
    \caption{Return curves of POAM policies.}
    \label{fig:smac-return-poam}
  \end{subfigure}
  \hfill
\begin{subfigure}[t]{0.48\textwidth}
    \includegraphics[width=\textwidth]{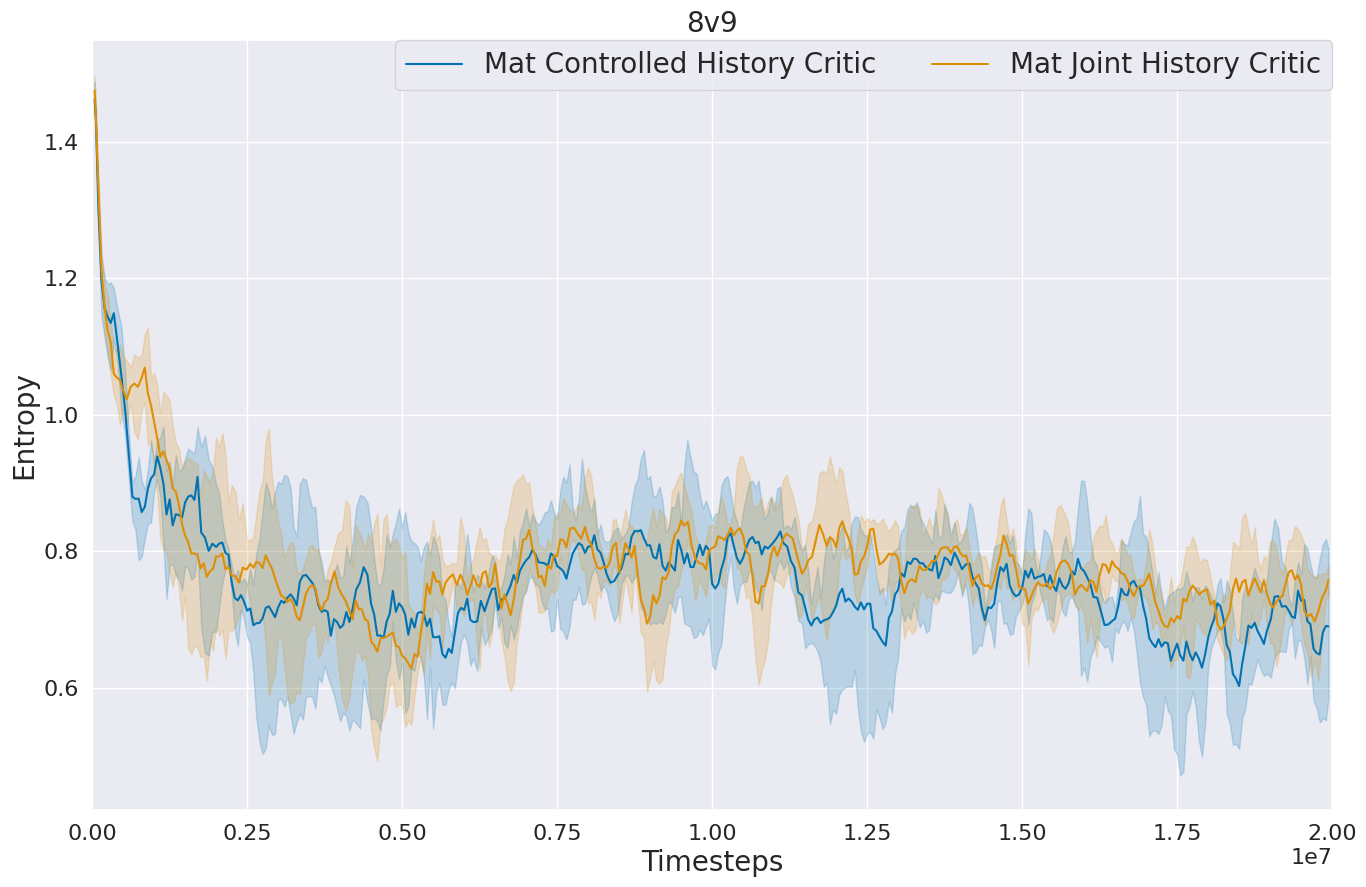}
    \caption{Entropy curves of MAT policies \caroline{This figure makes it look like there's little difference between controlled and joint history critics. I suggest instead visualizing the difference between the returns of the two methods.}}
    \label{fig:smac-mat}
  \end{subfigure}
  \hfill
  \begin{subfigure}[t]{0.48\textwidth}
    \includegraphics[width=\textwidth]{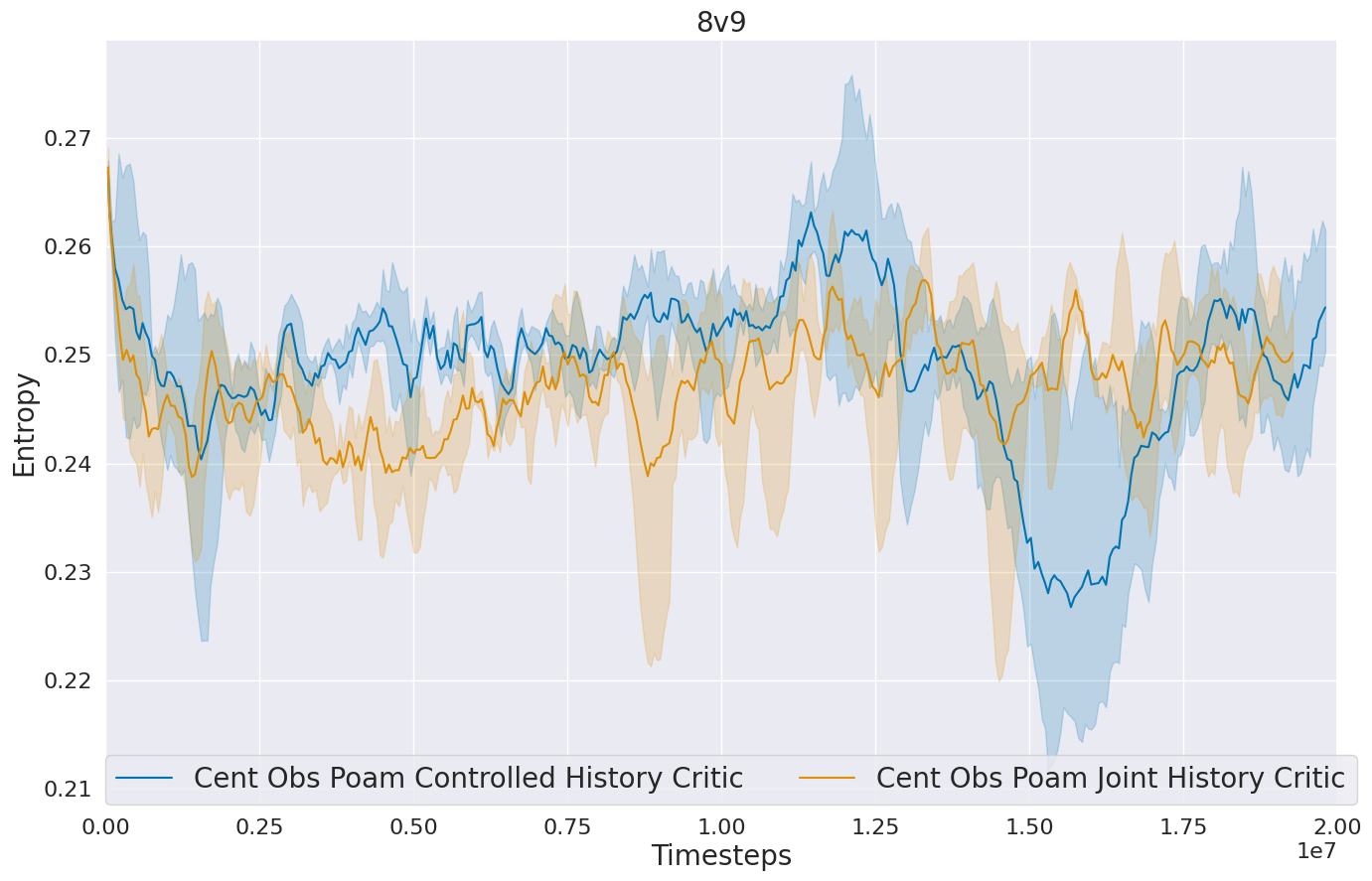}
    \caption{Entropy curves of POAM policies.}
    \label{fig:smac-poam}
  \end{subfigure}
  \caption{Comparison of SMAC return curves and out of entropy curves across MAT and POAM variants.}
  \label{fig:smac-experiments}
\end{figure}

\begin{table}[ht]
\centering
\caption{8v9 POAM Hypothesis Tests.}
\label{tab:smac-paired}
\begin{tabular}{lrr}
\hline
Target algorithm & Mean difference (C--J) & Paired $p$-value \\
\hline
ippo$^{***}$  & -0.765862	 & 0.000775	\\
iql$^{**}$  & -0.589716 & 0.005665	 \\
mappo$^{**}$  & -0.539812 &0.008801	 \\
qmix$^{**}$   & -1.003580 & 0.001541	 \\
vdn$^{***}$   & -0.74291 & 0.000558	 \\
\hline
\end{tabular}
\end{table}

Results for the 8v9 SMAC challenge are similar to stag hunt. We see that joint critics perform better than or equal to controlled critics, and joint critics show improved sampling efficiency for the POAM policies. Hypothesis tests for out of distribution performance were not statistically significant for MAT policies, but showed a certain improvement for joint critic POAM policies.

\section{Conclusion}
We studied the design of centralized critics for $N$-agent ad hoc teamwork, focusing on whether conditioning the critic on the histories of uncontrolled teammates is beneficial. We compared two natural advantage estimators: a \emph{controlled history} estimator that conditions only on the learner's own observation-action histories, and a \emph{joint history} estimator that additionally conditions on the uncontrolled teammates' histories.

On the theoretical side, we proved that both estimators are unbiased for the controlled advantage, and we derived a variance decomposition that isolates the primary force governing whether the joint history estimator reduces variance: the benefit of subtracting a critic correlated to the action-value function as a baseline. We characterized the conditions under which the baseline-cancellation effect dominates---namely, when the joint $Q$ and $V$ functions are highly correlated across uncontrolled histories and have comparable magnitudes of variation---and illustrated both regimes through a simple coordination game in which the joint baseline can either reduce or increase variance depending on the policy under analysis. We made an argument for the joint history estimator having lower variance on average, providing some intuition for our experimental results.

Empirically, joint history critics match or outperform controlled history critics across both the coordination game and the Stag Hunt environment, for both the POAM and MAT base algorithms. In the coordination game, joint critics reach near-optimal policies markedly faster. In Stag Hunt, joint critics yield improved sample efficiency during training and higher returns against uncontrolled agents at evaluation, with statistically significant gains across most uncontrolled agent types tested.

Taken together, our theoretical analysis and experimental results provide both a justification and a practical recommendation: when training a centralized critic in the NAHT setting, conditioning on uncontrolled teammate histories is generally beneficial, and the variance decomposition we derive offers a principled lens for understanding when and why this benefit arises.

\paragraph{Limitations and future work.} Our variance analysis assumes access to unbiased value functions; extending these results to learned critics, where approximation error interacts with the variance trade-off, is a natural next step. Furthermore, our Visitation Weighted
Advantage Variance argument is also currently a conjecture, supported by experimental evidence but not formally established; proving conditions under which the visitation-weighted variance is provably lower under the joint critic would strengthen the theoretical picture. We also restrict our analysis to the variance of the advantage function, but do not connect the ideas that a lower variance results in greater sampling efficiency. One way to conduct this analysis might be by looking at the variance of the policy gradient.

\bibliographystyle{plainnat}
\bibliography{ref}

\appendix




\section{Counterexample: When the Joint Baseline Increases Variance}
We now modify the example in Section~\ref{sec:coordination-example} to illustrate a case where the joint history baseline \emph{increases} variance over the controlled critic. We retain the setup from before, but flip the controlled policy: $\{0,0\}$ is now played with probability $p_0 = 0.2$ and $\{1,1\}$ with probability $p_1 = 0.8$. We again analyze the advantage of the action $a_t = \{0,0\}$.
\label{AppendixA} 

\paragraph{Value functions.} Recomputing the per-step expected rewards under each uncontrolled policy:
\begin{align}
    \mathbb{E}[r \mid \text{Mostly 0}] &= 0.2(0.6561) + 0.8(0.0001) = 0.1313 \\
    \mathbb{E}[r \mid \text{Mostly 1}] &= 0.2(0.00391) + 0.8(0.3164) = 0.2540
\end{align}
Marginalizing, $\mathbb{E}[r] = 0.5(0.1313) + 0.5(0.2540) = 0.1926$. The controlled advantage of $a_t = \{0,0\}$ is unchanged from before in expected one-step reward, $\mathbb{E}[r \mid a_t = \{0,0\}] = 0.330$, so:
\begin{equation}
    A(h_t^c, \{0,0\}) = 0.330 - 0.1926 = 0.1374
\end{equation}

\paragraph{Advantage estimators.} As before, $\hat{A}_{controlled} = r_t - 0.1926$, and
\begin{equation}
    \hat{A}_{joint} \mid \text{Mostly 0} = r_t - 0.1313, \qquad \hat{A}_{joint} \mid \text{Mostly 1} = r_t - 0.2540
\end{equation}
Both are unbiased for $A(h_t^c, \{0,0\}) = 0.1374$.

\paragraph{Variance computation.} The outcome probabilities are unchanged from the previous example, since they depend only on the uncontrolled policy and the conditional reward distribution given $a_t = \{0,0\}$. Computing variances as expected squared deviations from the mean:
\begin{align}
    \operatorname{Var}\!\big[\hat{A}_{controlled}\big] &= 0.330 \cdot (0.8074 - 0.1374)^2 + 0.670 \cdot (-0.1926 - 0.1374)^2 \nonumber \\
    &= 0.330 \cdot 0.4489 + 0.670 \cdot 0.1089 \approx 0.221
\end{align}
\begin{align}
    \operatorname{Var}\!\big[\hat{A}_{joint}\big] &= 0.3281 \cdot (0.8687 - 0.1374)^2 + 0.1720 \cdot (-0.1313 - 0.1374)^2 \nonumber \\
    &\quad + 0.00195 \cdot (0.7460 - 0.1374)^2 + 0.4980 \cdot (-0.2540 - 0.1374)^2 \nonumber \\
    &\approx 0.281
\end{align}

\paragraph{Result.} The joint history baseline now \emph{increases} the variance of the advantage estimator:
\begin{equation}
    \operatorname{Var}\!\big[\hat{A}_{controlled}\big] \approx 0.221, \qquad \operatorname{Var}\!\big[\hat{A}_{joint}\big] \approx 0.281
\end{equation}
The intuition mirrors the previous example. The action $\{0,0\}$ is now played with low probability and yields high reward precisely when paired with the uncontrolled policy (Mostly 0) under which the controlled policy performs \emph{poorly}. Conditioning on the uncontrolled history therefore amplifies, rather than absorbs, the reward fluctuations driven by $r_t$, and the joint baseline ends up tracking variation that the controlled baseline averaged out.


\section{Proofs of the Main Results}
\label{app:proofs}

\subsection{Unbiasedness of Advantage Estimators}
\label{app:proof-unbiased}
\begin{proof}[Proof of Theorem~\ref{theorem:unbiased}]
For the controlled estimator, since $V(h_t^c)$ is constant under the conditioning on $h_t^c$:
\begin{align}
    \mathbb{E}_{h_t^u,\, \tau \sim \pi \mid h_t^c, a_t}\!\big[\hat{A}_{controlled}\big] &= \mathbb{E}_{h_t^u,\, \tau \sim \pi \mid h_t^c, a_t}\!\big[\hat{Q}_t^{(1)}\big] - \mathbb{E}_{h_t^u,\, \tau \sim \pi \mid h_t^c}\!\big[\hat{V}(h_t^c)\big] \nonumber \\
    &= Q(h_t^c, a_t) - V(h_t^c) \nonumber \\
    &= A(h_t^c, a_t)
\end{align}
For the joint history estimator:
\begin{equation}
    \mathbb{E}_{h_t^u,\, \tau \sim \pi \mid h_t^c, a_t}\!\big[\hat{A}_{joint}\big] = \mathbb{E}_{h_t^u,\, \tau \sim \pi \mid h_t^c, a_t}\!\big[\hat{Q}_t^{(2)}\big] - \mathbb{E}_{h_t^u \mid h_t^c, a_t}\!\big[V(h_t^c, h_t^u)\big]
\end{equation}
Since the policy $\pi(a_t \mid h_t^c)$ does not depend on $h_t^u$, we have $p(h_t^u \mid h_t^c, a_t) = p(h_t^u \mid h_t^c)$, and therefore:
\begin{equation}
    \mathbb{E}_{h_t^u \mid h_t^c, a_t}\!\big[V(h_t^c, h_t^u)\big] = \mathbb{E}_{h_t^u \mid h_t^c}\!\big[V(h_t^c, h_t^u)\big] = V(h_t^c)
\end{equation}
Substituting, and again using the unbiasedness of $\hat{Q}_t^{(2)}$:
\begin{equation}
    \mathbb{E}_{h_t^u,\, \tau \sim \pi \mid h_t^c, a_t}\!\big[\hat{A}_{joint}\big] = Q(h_t^c, a_t) - V(h_t^c) = A(h_t^c, a_t)
\end{equation}
\end{proof}

\subsection{Advantage Variance Decomposition}
\label{app:proof-decomp}
\begin{proof}[Proof of Theorem~\ref{theorem:decomp}]
For the controlled estimator, expanding the variance gives:
\begin{align}
    &\operatorname{Var}_{h_t^u,\, \tau \sim \pi \mid h_t^c, a_t}\!\big[\hat{A}_{\text{controlled}}\big]
    = \\
    &\operatorname{Var}_{h_t^u,\, \tau \sim \pi \mid h_t^c, a_t}\!\big[\hat{Q}_t^{(1)}\big]
    + \operatorname{Var}_{h_t^u,\, \tau \sim \pi \mid h_t^c, a_t}\!\big[\hat{V}(h_t^c)\big]
    - 2\operatorname{Cov}_{h_t^u,\, \tau \sim \pi \mid h_t^c, a_t}\!\big[\hat{Q}_t^{(1)}, \hat{V}(h_t^c)\big]
\end{align}
By Assumption~\ref{ass:deterministic-value}, $\operatorname{Var}_{h_t^u,\, \tau \sim \pi \mid h_t^c}\!\big[\hat{V}(h_t^c) \big] = 0$, so $\operatorname{Cov}_{h_t^u,\, \tau \sim \pi \mid h_t^c, a_t}\!\big[\hat{Q}_t^{(1)}, \hat{V}(h_t^c)\big] = 0$, and  we are left with
\begin{equation}
    \operatorname{Var}_{h_t^u,\, \tau \sim \pi \mid h_t^c, a_t}\!\big[\hat{A}_{\text{controlled}}\big] = \operatorname{Var}_{h_t^u,\, \tau \sim \pi \mid h_t^c, a_t}\!\big[\hat{Q}_t^{(1)}\big]
\end{equation}
For the joint history estimator:
\begin{align}
    \operatorname{Var}_{h_t^u,\, \tau \sim \pi \mid h_t^c, a_t}\!\big[\hat{A}_{joint}\big] &= \operatorname{Var}_{h_t^u,\, \tau \sim \pi \mid h_t^c, a_t}\!\big[\hat{Q}_t^{(2)}\big] + \operatorname{Var}_{h_t^u,\, \tau \sim \pi \mid h_t^c, a_t}\!\big[\hat{V}(h_t^c, h_t^u)\big] \nonumber \\
    &\quad - 2\operatorname{Cov}_{h_t^u,\, \tau \sim \pi \mid h_t^c, a_t}\!\big[\hat{Q}_t^{(2)},\, \hat{V}(h_t^c, h_t^u)\big]
\end{align}
Taking the difference:
\begin{align}
    &\operatorname{Var}_{h_t^u,\, \tau \sim \pi \mid h_t^c, a_t}\!\big[\hat{A}_{controlled}\big] - \operatorname{Var}_{h_t^u,\, \tau \sim \pi \mid h_t^c, a_t}\!\big[\hat{A}_{joint}\big] \nonumber \\
    &\quad = \operatorname{Var}_{h_t^u,\, \tau \sim \pi \mid h_t^c, a_t}\!\big[\hat{Q}_t^{(1)}\big] - \operatorname{Var}_{h_t^u,\, \tau \sim \pi \mid h_t^c, a_t}\!\big[\hat{Q}_t^{(2)}\big] \nonumber \\
    &\qquad + 2\operatorname{Cov}_{h_t^u,\, \tau \sim \pi \mid h_t^c, a_t}\!\big[\hat{Q}_t^{(2)},\, \hat{V}(h_t^c, h_t^u)\big] - \operatorname{Var}_{h_t^u,\, \tau \sim \pi \mid h_t^c, a_t}\!\big[\hat{V}(h_t^c, h_t^u)\big]
\end{align}
We now simplify the last two terms. By Assumption~\ref{ass:deterministic-value}, $\hat{V}(h_t^c, h_t^u)$ is a deterministic function of $(h_t^c, h_t^u)$, so all variability in $\hat{V}(h_t^c, h_t^u)$ comes from $h_t^u$:
\begin{equation}
    \operatorname{Var}_{h_t^u,\, \tau \sim \pi \mid h_t^c, a_t}\!\big[\hat{V}(h_t^c, h_t^u)\big] = \operatorname{Var}_{h_t^u \mid h_t^c, a_t}\!\big[\hat{V}(h_t^c, h_t^u)\big]
\end{equation}
For the covariance, we apply the law of total covariance, conditioning on $h_t^u$:
\begin{align}
    &\operatorname{Cov}_{h_t^u,\, \tau \sim \pi \mid h_t^c, a_t}\!\big[\hat{Q}_t^{(2)},\, \hat{V}(h_t^c, h_t^u)\big] \nonumber \\
    &\quad = \underbrace{\mathbb{E}_{h_t^u \mid h_t^c, a_t}\!\Big[\operatorname{Cov}_{\tau \sim \pi \mid h_t^c, h_t^u, a_t}\!\big[\hat{Q}_t^{(2)},\, \hat{V}(h_t^c, h_t^u)\big]\Big]}_{= \, 0 \text{ by Assumption }\ref{ass:deterministic-value}} \nonumber \\
    &\qquad + \operatorname{Cov}_{h_t^u \mid h_t^c, a_t}\!\Big[\mathbb{E}_{\tau \sim \pi \mid h_t^c, h_t^u, a_t}\!\big[\hat{Q}_t^{(2)}\big],\, \mathbb{E}_{\tau \sim \pi \mid h_t^c, h_t^u, a_t}\!\big[\hat{V}(h_t^c, h_t^u)\big]\Big] \nonumber \\
    &\quad = \operatorname{Cov}_{h_t^u \mid h_t^c, a_t}\!\big[q^{(2)}(h_t^c, h_t^u, a_t),\, \hat{V}(h_t^c, h_t^u)\big]
\end{align}
where in the last step follows from Assumption~\ref{ass:deterministic-value}. Substituting both simplifications:
\begin{align}
    &\operatorname{Var}_{h_t^u,\, \tau \sim \pi \mid h_t^c, a_t}\!\big[\hat{A}_{controlled}\big] - \operatorname{Var}_{h_t^u,\, \tau \sim \pi \mid h_t^c, a_t}\!\big[\hat{A}_{joint}\big] \nonumber \\
    &\quad = \operatorname{Var}_{h_t^u,\, \tau \sim \pi \mid h_t^c, a_t}\!\big[\hat{Q}_t^{(1)}\big] - \operatorname{Var}_{h_t^u,\, \tau \sim \pi \mid h_t^c, a_t}\!\big[\hat{Q}_t^{(2)}\big] \nonumber \\
    &\qquad + 2\operatorname{Cov}_{h_t^u \mid h_t^c, a_t}\!\big[q^{(2)}(h_t^c, h_t^u, a_t),\, \hat{V}(h_t^c, h_t^u)\big] - \operatorname{Var}_{h_t^u \mid h_t^c, a_t}\!\big[\hat{V}(h_t^c, h_t^u)\big]
\end{align}
To obtain the final form, we apply the identity $2\operatorname{Cov}[X, Y] - \operatorname{Var}[Y] = \operatorname{Var}[X] - \operatorname{Var}[X - Y]$, with $X = q^{(2)}(h_t^c, h_t^u, a_t)$ and $Y = \hat{V}(h_t^c, h_t^u)$:
\begin{align}
    &\operatorname{Var}_{h_t^u,\, \tau \sim \pi \mid h_t^c, a_t}\!\big[\hat{A}_{controlled}\big] - \operatorname{Var}_{h_t^u,\, \tau \sim \pi \mid h_t^c, a_t}\!\big[\hat{A}_{joint}\big] \nonumber \\
    &\quad = \operatorname{Var}_{h_t^u,\, \tau \sim \pi \mid h_t^c, a_t}\!\big[\hat{Q}_t^{(1)}\big] - \operatorname{Var}_{h_t^u,\, \tau \sim \pi \mid h_t^c, a_t}\!\big[\hat{Q}_t^{(2)}\big] \nonumber \\
    &\qquad + \operatorname{Var}_{h_t^u \mid h_t^c, a_t}\!\big[q^{(2)}(h_t^c, h_t^u, a_t)\big] - \operatorname{Var}_{h_t^u \mid h_t^c, a_t}\!\big[q^{(2)}(h_t^c, h_t^u, a_t) - \hat{V}(h_t^c, h_t^u)\big]
\end{align}
\end{proof}

\end{document}